\documentclass[letterpaper]{article} 
\usepackage{aaai24}  
\usepackage{times}  
\usepackage{helvet}  
\usepackage{courier}  
\usepackage[hyphens]{url}  
\usepackage{graphicx} 
\usepackage{natbib}  
\usepackage{caption} 
\usepackage[ruled,vlined,linesnumbered]{algorithm2e}

\usepackage{amssymb}
\usepackage{stmaryrd}
\usepackage{amsmath}
\usepackage{enumitem}

\newenvironment{proof}{\begin{trivlist}\item\noindent{\sc Proof.}}{\qed\end{trivlist}}
\newcommand{\qed}{\hfill {\footnotesize$\square$\hspace{0mm}}}

\renewcommand{\[}{\llbracket}
\renewcommand{\]}{\rrbracket}
\renewcommand{\S}{{\sf S}}
\newcommand{\C}{{\sf C}}
\newcommand{\Gc}{{\sf G}^{\sf c}}
\newcommand{\Gs}{{\sf G}^{\sf s}}
\newcommand{\Gcs}{{\sf G}^{\sf c,s}}
\renewcommand{\phi}{\varphi}

\newtheorem{theorem}{Theorem}
\newtheorem{corollary}{Corollary}

\newtheorem{lemma}{Lemma}
\newtheorem{definition}{Definition}
\newtheorem{property}{Property}

\usepackage{subcaption}

\title{Responsibility in Extensive Form Games}
\author {
    Qi Shi
    \thanks{Many thanks to Dr Pavel Naumov, PhD supervisor of Qi Shi, for his selfless help in finishing this paper. The research is funded by the China Scholarship Council (CSC No.202206070014).}
}
\affiliations{
    University of Southampton\\
    qi.shi@soton.ac.uk
}

\begin{document}

\maketitle

\begin{abstract}
Two different forms of responsibility, counterfactual and seeing-to-it, have been extensively discussed in the philosophy and AI in the context of a single agent or multiple agents acting simultaneously. Although the generalisation of counterfactual responsibility to a setting where multiple agents act in some order is relatively straightforward, the same cannot be said about seeing-to-it responsibility. Two versions of seeing-to-it modality applicable to such settings have been proposed in the literature. Neither of them perfectly captures the intuition of responsibility. This paper proposes a definition of seeing-to-it responsibility for such settings that amalgamate the two modalities.

This paper shows that the newly proposed notion of responsibility and counterfactual responsibility are not definable through each other and studies the responsibility gap for these two forms of responsibility. It shows that although these two forms of responsibility are not enough to ascribe responsibility in each possible situation, this gap does not exist if higher-order responsibility is taken into account.
\end{abstract}

\section{Introduction}\label{sec:introduction}

In the United States, if a person is found guilty by a state court and all appeals within the state justice system  have been exhausted, the person can petition the Governor of the state for executive clemency. The US Supreme Court once described the clemency by the executive branch of the government as the ``fail safe'' of the criminal justice system~\cite{herrera93case}. 
This was the case with Barry Beach, who was found guilty of killing a 17-year-old high school valedictorian Kim Nees and sentenced in 1984 to 100 years imprisonment without parole~\cite{ap15guardian}.
In 2014, after a court appeal, a retrial, and a negative decision by the Montana Supreme Court, Barry's attorney filed a petition for executive clemency.

To prevent corruption and favouritism by the Governor, many states in the US have boards that must support the decision before the Governor can grant executive clemency. In Montana, such a board has existed since the original 1889 Constitution~\cite[Article VII, Section 9]{cc89url}. 
With time, the law, the name of the board, and the way it grants approval changed~\cite{bopp23url},
but the Board maintained the ability to constrain the Governor's power to grant executive clemency. The executive clemency procedure that existed in Montana by 2014 is captured by the {\em extensive form game} depicted in Figure~\ref{fig:board and governor A}. First, the Board (agent $b$) can either deny (action D) the clemency or recommend (E) it. If the Board recommends, then the Governor (agent $g$) might grant (G) or not grant (F) the executive clemency.  

\begin{figure}
    \centering
    \begin{subfigure}[b]{0.2\textwidth}
        \begin{center}
        \scalebox{0.5}{\includegraphics{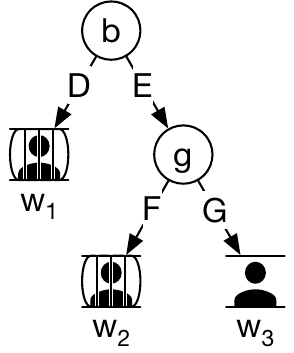}}
        \caption{Montana situation}
       \label{fig:board and governor A}
        \end{center}
    \end{subfigure}
    \begin{subfigure}[b]{0.2\textwidth}
        \begin{center}
        \scalebox{0.5}{\includegraphics{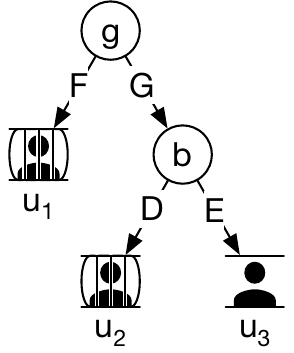}}
        \caption{Switched action order}
        \label{fig:board and governor B}
        \end{center}
    \end{subfigure}
    \caption{Executive clemency procedure}
    \label{fig:board and governor AB}
\end{figure}
The executive clemency procedure in Montana is a typical multiagent system where the final outcome is determined by the decisions of all agents in the system. Such multiagent systems widely exist in both human and machine activities and have been studied from multiple perspectives. \emph{Responsibility} is one of the topics in those studies. 
Although there is no commonly acknowledged definition of responsibility, we usually have some vague intuition about the responsibility of an agent when we think of the agent being praiseworthy for a positive result or blameworthy for a negative result, especially if the agent is undertaking moral or legal obligations. 
It is also usually assumed that responsibility is connected to \emph{free will} of an agent to act. Note that, an agent can act to prevent or to achieve a certain result. This gives rise to two forms of responsibility commonly considered in the literature: {\em counterfactual responsibility} and {\em responsibility for seeing to it}, respectively.
In this paper, I study these two forms of responsibility in the multiagent systems that can be modelled as extensive form games.

The rest of this paper is divided into four major sections. In Section~\ref{sec:literature}, I give a review of the two forms of responsibility and related logic notions in the literature. 
Based on the discussion, I propose a new form of seeing-to-it responsibility for extensive form game settings in Section~\ref{sec:our notion}. 
Then, I formally define the model of extensive form games in Section~\ref{sec:terminology} and the syntax and semantics of the two forms of responsibility in Section~\ref{sec:syntax and semantics}. In particular, I show the mutual undefinability between the two forms of responsibility, discuss the meaning of higher-order responsibility, and state the complexity of model checking. 
Finally, in Section~\ref{sec:gap}, I formally study the {\em responsibility gap}.
Although discussion of the responsibility gap is prevalent in the literature~\cite{matthias2004responsibility,braham2011responsibility,duijf2018responsibility,burton2020mind,gunkel2020mind,langer2021we,goetze2022mind}, only a few studies~\cite{braham2018voids,hiller2022axiomatic} give a formal definition of the concept. I then define the hierarchy of responsibility gaps, which, as far as I know, have never been discussed before. I show that a higher-order responsibility gap does not exist for sufficiently high orders.

\section{Literature review and notion discussion}\label{sec:literature}

Counterfactual responsibility captures the {\em principle of alternative possibilities}~\cite{f69tjop,belnap1992way,w17}: {\em an agent is responsible for a statement $\phi$ in an outcome if $\phi$ is true in the outcome and the agent had a strategy that could prevent it}. For example, consider outcome $w_3$ in Figure~\ref{fig:board and governor A}, where the Board recommends (E) clemency and the Governor grants (G) it. In this case, Beach is set free. Note that both the Board and the Governor have a strategy (action D for the Board, action F for the Governor) to prevent this. As a result, each of them is counterfactually responsible for the fact that Beach, who was found by the court to be the murderer of Kim Nees, escapes punishment. 

Next, consider outcome $w_2$ in which the Board recommends (E) clemency, but the Governor does not grant (F) it. Beach is left in prison. In this case, the Board is not counterfactually responsible for the fact that Beach is left in prison because the Board had no strategy to prevent this. At the same time, the Governor had such a strategy (action G). As a result, the Governor is counterfactually responsible for the fact that Beach is left in prison in outcome $w_2$. 

Note that in order for Beach to be freed, both the Governor and the Board must support this. However, from the point of view of ascribing counterfactual responsibility, the order in which the decisions are made is important. If the Governor acts first, then, essentially, the roles of the Governor and the Board switch, see Figure~\ref{fig:board and governor B}. In this new situation, the Governor is no long counterfactually responsible for Beach being left in prison because he no longer has a strategy to prevent this.
The dependency on the order of the decisions makes counterfactual responsibility in extensive form games different from the previously studied counterfactual responsibility in {\em strategic} game settings~\cite{ls11ai,naumov2019blameworthiness,naumov2020epistemic}, where all agents act concurrently and just once. 
The above definition of counterfactual responsibility for extensive form games is introduced in~\cite{yazdanpanah2019strategic}.
It also appears in~\cite{baier2021game}.

The other commonly studied form of responsibility is defined through the notion of {\em seeing-to-it}. As a modality, seeing-to-it has been well studied in STIT logic~\cite{chellas1969logical,belnap1990seeing,h01,hp17rsl}. Informally, an agent sees to it that $\phi$ if the agent guarantees that $\phi$ happens.
When using the notion of seeing-to-it to define a form of responsibility, a \textit{negative condition}\footnote{In the general STIT models, a negative condition is a history where $\neg\phi$ is true~\cite{p91synthese}. In the extensive form game settings, a negative condition is an {\em outcome} where $\neg\phi$ is true. } is usually required to exist to capture the intuition that no agent should be responsible for a trivial truth such as ``$1+1=2$''.
The notion of \textit{deliberative} seeing-to-it~\cite{horty1995deliberative,x98jpl,bht08jpl,ow16sl} captures this idea by adding the requirement of a negative condition. 
Some follow-up work such as \cite{lorini2014logical} and \cite{abarca2022stit} further incorporates the epistemic states of the agents into their discussion, but this is still within the STIT frame.
\citet{naumov2021two} studied deliberative seeing-to-it as one of the forms of responsibility in strategic game settings. 


In extensive form game settings, there are two versions of the notion of seeing-to-it that may {\em potentially} capture a form of responsibility: \textit{strategically} seeing-to-it in the presence of a negative condition and
\textit{achievement} seeing-to-it.

\emph{Strategically seeing-to-it}~\cite{bht06jelia,broersen2009stit} is defined under the assumption that each agent commits upfront to a strategy (\textit{i.e.} a plan of actions) for the duration of the game. Instead of guaranteeing $\phi$ to happen with one action, such a strategy guarantees $\phi$ to happen \emph{in the final outcome} after acting according to the strategy in the whole game, no matter how the other agents may act in the process.
For example, in the game depicted in Figure~\ref{fig:board and governor A}, both the Board and the Governor have an upfront strategy to leave Beach in prison. For the Board, the strategy consists in denying the petition. For the Governor, the strategy consists in waiting for the Board to act and, if the Board recommends clemency, rejecting the petition. 
By incorporating the notion of a strategy, strategically seeing-to-it in the presence of a negative condition can be treated as a natural extension of deliberative seeing-to-it in multi-step decision schemes such as extensive form games.

However, this ``natural extension'' does not work for two reasons. on the one hand, by definition, the notion of strategically seeing-to-it has to be evaluated based on strategies rather than outcomes~\cite{broersen2015using}.
However, in some applications, such strategies may not be observable. Let us consider the case of outcome $w_1$ in Figure~\ref{fig:board and governor A}. Here, the strategy of the Governor is not observable because he has no chance to make any choice. No one except for the Governor himself can tell how he would choose if the Board had not denied the clemency. Hence, even though he has a strategy to guarantee Beach being left in prison and the strategy is followed {\em in a trivial way} in outcome $w_1$, it is still not clear whether the Governor strategically sees to Beach being left in prison or not.
On the other hand, although the strategy can be observed in some cases (such as pre-programmed autonomous agents), the notion of strategically seeing-to-it accuses the agents of {\em thoughtcrime} purely based on their plans rather than actions. Note that, in law, {\em actus reus} (``guilty actions'') is a commonly required element of a crime \cite{e21sep}. 
For this reason, even if the Governor's strategy is to deny the clemency when the Board recommends it, which indeed strategically sees to Beach being left in prison according to the definition, the Governor should not be held {\em responsible} for seeing to this in outcome $w_1$, since he takes no action at all.
Therefore, \emph{strategically seeing-to-it in the presence of a negative condition} cannot always serve as a proper notion of responsibility in extensive form games.

Another notion of seeing-to-it that may capture a form of responsibility in extensive form game settings is \emph{achievement seeing-to-it}~\cite{belnap1992way,horty1995deliberative}. In an extensive form game, the agents make choices one after another. Each choice of the agents may eliminate the possibility of some outcomes until the final outcome remains. If a statement is true in the final outcome, then during the game process, all the negative conditions, if exist, are eliminated.
Achievement seeing-to-it captures the idea that, in such multi-step decision schemes, one specific \emph{choice} of an agent guarantees some statement to be true in the final outcome by eliminating the ``last possibility'' for a negative condition to be achieved. For example, in outcome $w_3$ of Figure~\ref{fig:board and governor A}, Beach is set free after the Board recommends (E) the clemency and the Governor grants (G) it. The choice of the Board (action E) eliminates one possibility of a negative condition ($w_1$) and the choice of the Governor (action G) eliminates the other possibility of a negative condition ($w_2$), which is also the last possibility. Hence, the Governor sees to it that Beach is set free in the achievement way in outcome $w_3$. Note that the notion of achievement seeing-to-it implies the existence of a negative condition by itself.

Achievement seeing-to-it can be treated as a form of responsibility in an intuitive sense. However, this notion cannot capture the idea of ``guaranteeing'' when regarding the extensive form game as a whole process. Let us still consider outcome $w_3$ in Figure~\ref{fig:board and governor A}. When we treat the executive clemency procedure as a whole, the Governor does {\em not} guarantee that Beach will be set free, since the Board could have chosen to deny (D) the clemency before the Governor can make any decision. In fact, the Governor does not even have the {\em ability} to guarantee that Beach will be set free. Therefore, it is hard to say that the Governor is responsible for ``seeing to it that'' Beach is set free in outcome $w_3$, even though he sees to this in the achievement way.

The inconsistency between the notion of achievement seeing-to-it and the seeing-to-it form of responsibility is more significant when \emph{obligation} is taken into consideration. For example, the obligation of doctors is to try their best to cure their patients. Consider a situation where a patient in danger of life is waiting for treatment. Suppose the treatment is sure to cure the patient. But the doctor leaves the patient unattended for six days and gives treatment on the seventh day. Then, the patient is cured. By giving the treatment, the doctor sees to it that the patient is cured in the achievement way. However, the doctor cannot be said to ``be responsible (praiseworthy) for seeing to it that'' the patient would be cured, because the patient might have died at any time during the first six days. For this reason, \emph{achievement seeing-to-it} often cannot serve as a proper notion of the responsibility for seeing to it in extensive form games.

\section{My notion of responsibility for seeing to it}\label{sec:our notion}

In this section, I introduce a new notion of seeing-to-it responsibility that fits into the extensive form game setting.

{\em First}, I modify the notion of strategically seeing-to-it into a backward version. I would say that an agent {\em backwards-strategically} sees to $\phi$ if the agent has an upfront ability to guarantee that $\phi$ would be true in the outcome and maintains the ability for the duration of the game. The ability to guarantee $\phi$ is captured by the {\em existence of a strategy} that guarantees $\phi$. Note that, although the maintenance of the ability can be achieved by following such a strategy, the backward version of strategically seeing-to-it does {\em not} require the actually applied strategy to guarantee $\phi$. 
Intuitively, instead of caring about the {\em plan} of the agent to guarantee $\phi$, the backward version of strategically seeing-to-it focuses on the {\em ability} of guaranteeing it. 

Unlike the original notion of strategically seeing-to-it, the backward version can be evaluated based on the outcomes (the paths of play) in extensive form game settings.
For example, observe that in the game depicted in Figure~\ref{fig:board and governor A}, the Board has the ability (the existence of a strategy) to guarantee that Beach would be left in prison at the $b$-labelled node and outcomes $w_1$ and $w_2$. The Governor has the same ability at the $b$-labelled node, the $g$-labelled node, and outcomes $w_1$ and $w_2$. On the path of play toward outcome $w_1$, both the Board and the Governor maintain this ability. Hence, in outcome $w_1$, both the Board and the Governor backwards-strategically see to Beach being left in prison.
Note that $w_1$ is the outcome when the Governor applies the strategy ``to grant (G) the clemency if the Board recommend (E) it'' and the Board applies the strategy ``to deny (D) the clemency''. The Governor's strategy does {\em not} strategically see to Beach being left in prison in the original meaning. However, he still backwards-strategically sees to it.
In outcome $w_2$, only the Governor backwards-strategically sees to Beach being left in prison because the Board loses the ability at the $g$-labelled node, where the Governor can grant (G) the clemency.

{\em Second}, I use the notion of backwards-strategically seeing-to-it, in combination with achievement seeing-to-it, to define the seeing-to-it form of responsibility in extensive form game settings. I would say that an agent is \textit{responsible} for seeing to $\phi$ if she sees to it both backwards-strategically and in the achievement way. This combination captures both the ability and the action to ``guarantee'' in the notion of seeing-to-it.
Informally, I say that in the extensive form games, an agent is responsible for seeing to $\phi$ if the agent {\em has an upfront ability to achieve $\phi$, maintains it throughout the game, and eliminates the last possibility of a negative condition in the process}. 

Consider the game depicted in Figure~\ref{fig:board and governor A}. In outcome $w_1$, the Board sees to Beach being left in prison both backwards-strategically and in the achievement way. Therefore, the Board is responsible for seeing to Beach being left in prison in $w_1$. 
This argument is also true for the Governor in outcome $w_2$. However, in $w_1$, the Governor sees to Beach being left in prison backwards-strategically but not in the achievement way. Thus, the Governor is not responsible for seeing to this in $w_1$. In $w_2$, the Board sees to Beach being left in prison neither backwards-strategically nor in the achievement way. Hence, the Board is not responsible for seeing to this in $w_2$.
Moreover, in outcome $w_3$, the Governor sees to Beach being set free in the achievement way but not backwards-strategically (he does not have such an ability at the $b$-labelled node). Thus, the governor is not responsible for seeing to Beach being set free in outcome $w_3$.

\section{Extensive form games terminology}\label{sec:terminology}

In this section, I introduce extensive form games that are used later to give formal semantics of my modal language. Throughout the paper, I assume a fixed set of agents $\mathcal{A}$ and a fixed nonempty set of propositional variables.

\begin{definition}\label{game}
   An extensive form game is a nonempty finite rooted tree in which each non-leaf node is labelled with an agent and each leaf node is labelled with a set of propositional variables.  
\end{definition}
I refer to the leaf nodes of a game as {\em outcomes} of the game. The set of all outcomes of a game $G$ is denoted by $\Omega(G)$. By $parent(n)$, I denote the parent node of any non-root node $n$. I write $n_1\preceq n_2$ if node $n_2$ is on the simple path (including ends) between the root node and node $n_1$. I say that an outcome is labelled with a propositional variable if the outcome is labelled with a set containing this propositional variable. 

\begin{definition}\label{df:achievement point}
For any set $X$ of outcomes and any agent $a$, non-root node $n$ is an $X$-achievement point by agent $a$, if
\begin{enumerate}
    \item $parent(n)$ is labelled with agent $a$;
    \item $w\notin X$ for some outcome $w$ such that $w\preceq parent(n)$;
    \item $w\in X$ for each outcome $w$ such that $w\preceq n$.
\end{enumerate}
\end{definition}

The notion of achievement point captures the idea that outcomes in set $X$ are already ``achieved'' by agent $a$ at node $n$: agent $a$ choosing $n$ at node $parent(n)$ eliminates the \textit{last} possibility for an outcome \textit{not} in $X$ to be realised and thus guarantees that the game will end in $X$. 
For example, in the extensive form game depicted in Figure~\ref{fig:board and governor A}, consider the set $\{w_1, w_2\}$ of outcomes where Beach is left in prison. Node $w_1$ is a $\{w_1, w_2\}$-achievement point by the Board, where action D of the Board at the $b$-labelled node eliminates the last possibility for Beach being set free ($w_3$) to come true. Similarly, node $w_2$ is a $\{w_1, w_2\}$-achievement point by the Governor. Note that an achievement point can also be a non-leaf node. For instance, the $g$-labelled node is a $\{w_2,w_3\}$-achievement point by the Board, since action E of the Board at the $b$-labelled node eliminates the last possibility for outcome $w_1$ to be realised. 
The next lemma shows a property of achievement point, whose significance is due to the uniqueness of the chance to eliminate the last possibility.

\begin{lemma}\label{lm:achievement point existence}
For any extensive form game $G$, any set of outcomes $X\subsetneq\Omega(G)$, and any outcome $w\in X$, there is a unique agent $a$ and a unique $X$-achievement point $n$ by agent $a$ such that $w\preceq n$. 
\end{lemma}

Next, I define the notation $win_a(X)$. For any set $X$ of outcomes and any agent $a$, by $win_a(X)$ I mean the set of all nodes (including outcomes) from which agent $a$ has the {\em ability} to end the game in set $X$. 
Formally, the set $win_a(X)$ is defined below using backward induction.
\begin{definition}\label{df:winning set}
For any set $X$ of outcomes, the set $win_a(X)$ is the minimal set of nodes such that
\begin{enumerate}
\item $X\subseteq win_a(X)$;
\item for any non-leaf node $n$ labelled with agent $a$, if {\bf\em at least one} child of node $n$ belongs to the set $win_a(X)$, then node $n\in win_a(X)$;
\item for any non-leaf node $n$ {\bf\em not} labelled with agent $a$, if {\bf\em all} children of node $n$ belong to the set $win_a(X)$, then node $n\in win_a(X)$.
\end{enumerate}
\end{definition}
Informally, for a non-leaf node $n\in win_a(X)$ labelled with agent $a$, the ability of agent $a$ to end the game in $X$ is captured by the strategy that always chooses a child node of $n$ from the set $win_a(X)$. Specifically, if the root of the tree is in the set $win_a(X)$, then agent $a$ has an upfront ability to end the game in $X$.

\section{Syntax and semantics}\label{sec:syntax and semantics}

The language $\Phi$ of my logical system is defined by the following grammar:
\begin{equation}
    \phi:= p\;|\;\neg\phi\;|\;\phi\wedge\phi\;|\;\C_a\phi\;|\;\S_a\phi, \notag
\end{equation}
where $p$ is a propositional variable and $a\in\mathcal{A}$ is an agent. I read $\C_a\phi$ as ``agent $a$ is counterfactually responsible for $\phi$'' and $\S_a\phi$ as ``agent $a$ is responsible for seeing to $\phi$''. I assume that Boolean constants true $\top$ and false $\bot$ are defined in the standard way.

The next is the core definition of this paper. 
Informally, for each formula $\phi\in\Phi$, the truth set $\[\phi\]$ is the set of all outcomes where $\phi$ is true.

\begin{definition}\label{sat} 
For any extensive form game $G$ and any formula $\phi\in\Phi$, the truth set $\[\phi\]$ is defined recursively:
    \begin{enumerate}
        \item $\[p\]$ is the set of all outcomes labelled  with propositional variable $p$;
        \item $\[\neg\phi\]=\Omega(G)\setminus \[\phi\]$;
        \item $\[\phi\wedge\psi\]=\[\phi\]\cap\[\psi\]$;
         \item $\[\C_a\phi\]$ is the set of all outcomes $w\in \[\phi\]$ such that there exists a node $n\in win_a(\[\neg\phi\])$ where $w \preceq n$;
        \item $\[\S_a\phi\]$ is the set of all outcomes $w\in\Omega(G)$ such that
        \begin{enumerate}
            \item $\{n\;|\; w\preceq n\}\subseteq win_a(\[\phi\])$;
            \item there exists a $\[\phi\]$-achievement point $n$ by agent $a$ such that $w\preceq n$.
        \end{enumerate}
    \end{enumerate}
\end{definition}
Item~4 above defines the notion of counterfactual responsibility following~\cite{yazdanpanah2019strategic} and \cite{baier2021game}. An agent $a$ is counterfactually responsible for a statement $\phi$ in outcome $w$ if two conditions are satisfied: (\romannumeral1) $\phi$ is true in $w$ and (\romannumeral2) on the path of play, agent $a$ has a strategy to prevent $\phi$. The first condition is captured by the assumption $w\in\[\phi\]$. The second condition is captured by the existence of a node $n$ on the path of play ($w\preceq n$) to outcome $w$ such that $n\in win_a(\[\neg\phi\])$.

Item~5 above defines my notion of the seeing-to-it form of responsibility as the combination of backwards-strategically seeing-to-it and achievement seeing-to-it. 
An agent backwards-strategically sees to $\phi$ in outcome $w$ if the agent has an upfront ability to achieve $\phi$ and maintains the ability throughout the game. This is captured by the fact that all the nodes on the path of play leading to outcome $w$ belong to the set $win_a(\[\phi\])$, as part 5(a) of Definition~\ref{sat} shows.
An agent sees to $\phi$ in the achievement way in outcome $w$ if the agent eliminates the last possibility for $\neg\phi$. This means, on the path of play toward outcome $w$, there exists a $\[\phi\]$-achievement point by agent $a$. This is captured in part~5(b) of Definition~\ref{sat}.


\subsection{Mutual undefinability between $\C$ and $\S$}\label{sec:undefinability}

Modalities $\C$ and $\S$ capture two forms of responsibility in extensive form games. One natural question is: do I need both of them? The answer would be no if one of them can be defined via the other.
For example, in propositional logic, implication $\to$ and disjunction $\vee$ can be defined via negation $\neg$ and conjunction $\wedge$.\footnote{ $\phi\to\psi\equiv\neg(\phi\wedge\neg\psi)$ and $\phi\vee\psi\equiv\neg(\neg\phi\wedge\neg\psi)$ for each formulae $\phi,\psi$ in propositional logic.} Hence, using only negation $\neg$ and implication $\to$ is enough for a propositional logic system.
Note that \citet{naumov2021two} proved that modality $\S$ is not definable via modality $\C$ but modality $\C$ is definable via modality $\S$ by $\C_a\phi\equiv \phi\wedge\S_a\neg\S_a\neg\phi$  in {\em strategic game settings}.
Before presenting my results in extensive form game settings, let us start with an auxiliary definition:

\begin{definition}\label{semantically equivalent}
Formulae $\phi,\psi\in \Phi$ are semantically equivalent if $\[\phi\]=\[\psi\]$ for each extensive form game.
\end{definition}

In language $\Phi$, modality $\C$ is {\em definable} via modality $\S$ if, for each formula $\phi\in\Phi$, there is a semantically equivalent formula $\psi\in\Phi$ that does {\em not} use modality $\C$. 
The definability of $\S$ via $\C$ could be specified similarly.
The two theorems below show that $\C$ and $\S$ are {\em not} definable via each other in extensive form game settings.

\begin{theorem}[undefinability of $\C$ via $\S$]\label{th:undefinability C via S}
    The formula $\C_a p$ is not semantically equivalent to any formula in language $\Phi$ that does not contain modality $\C$.
\end{theorem}

\begin{theorem}[undefinability of $\S$ via $\C$]\label{th:undefinability S via C}
    The formula $\S_a p$ is not semantically equivalent to any formula in language $\Phi$ that does not contain modality $\S$.
\end{theorem}


The formal proofs of the above two theorems can be seen in Appendix~\ref{app_sec:undefinability C via S} and \ref{app_sec:undefinability S via C}, respectively.
The undefinability results show that, in order to discuss both forms of responsibility in extensive form game settings, I need to include both modalities.

\subsection{Higher-order responsibility}\label{sec:higher-order responsibility}

As can be seen in the grammar of language $\Phi$, modalities $\C$ and $\S$ can be nested in a formula.
By higher-order responsibility, I mean more complicated forms of responsibility expressible by the nesting of modalities $\C$ and $\S$. For example, in outcome $w_2$ of the game depicted in Figure~\ref{fig:board and governor A}, the Governor is counterfactually responsible for Beach being left in prison. However, the Board could have {\em prevented such responsibility} by denying (D) the petition. Thus, in outcome $w_2$, the Board is counterfactually responsible for the Governor's responsibility for Beach being left in prison: 
$
w_2\in \[\C_b\C_g\text{``Beach is left in prison''}\].
$
Similarly, it is true that
$
w_2\in \[\C_b\S_g\text{``Beach is left in prison''}\].
$ 

Discussion of higher-order responsibility makes sense, especially in a situation where some of the agents who do affect the outcome are not the {\em proper subjects} to ascribe the responsibility. For example, young kids are usually not considered the proper subjects of criminal responsibility. Therefore, when they commit crimes and assume direct responsibility for the outcomes, the secondary responsibility of their guardians needs to be considered \cite{hollingsworth2007responsibility}. The same is true for autonomous agents and their designers.

There are some interesting properties of higher-order responsibility. For instance, {\em formulae $\C_a\C_a\phi$ and $\C_a\phi$ are semantically equivalent} (see Property~\ref{pp:CC=C} in Appendix~\ref{app_sec:CC=C}). This means, if an agent is counterfactually responsible for a statement $\phi$, then she is also counterfactually responsible for assuming this counterfactual responsibility.
Also, {\em both formulae $\S_b\S_a\phi$ and $\S_b\C_a\phi$ are semantically equivalent to $\bot$} (see Property~\ref{pp:SS empty} in Appendix~\ref{app_sec:SS=false} and Property~\ref{pp:SC empty} in Appendix~\ref{app_sec:SC=false}). This means there is no chance for an agent to be responsible for seeing to another agent's responsibility.

\subsection{Complexity of model checking}\label{sec:complexity}

In my setting, the computation of the set $\[\phi\]$ is the core of any model checking problem related to formula $\phi$. Hence, I analyse the complexity of computing the truth set $\[\phi\]$ for an arbitrary formula $\phi\in\Phi$. I assume that deciding whether an outcome is labelled with a propositional variable takes constant time. By Definition~\ref{df:achievement point}, Definition~\ref{df:winning set}, and Definition~\ref{sat}, I have the next theorem. See Appendix~\ref{app_sec:complexity} for detailed analysis.

\begin{theorem}[time complexity]\label{th:complexity}
For any formula $\phi\in\Phi$ and any extensive form game $G$, the computation of the set $\[\phi\]$ takes $O(|\phi|\cdot |G|)$, where $|\phi|$ is the size of formula $\phi$ and $|G|$ is the number of nodes in game $G$.
\end{theorem}

\section{Responsibility gap}\label{sec:gap}

One of the important questions discussed in the ethics literature is the responsibility gap. That is, if something happens, is there always an agent that can be held responsible for it? 
I now discuss if the two forms of responsibility considered in the paper are enough to avoid responsibility gaps in extensive form games. 
Note that, as I discussed in Section~\ref{sec:literature}, nobody should be responsible for a vacuous truth. Hence, in this section, I only consider the responsibility gaps for statements that are {\em not} trivially true.

Let us go back to the example depicted in Figure~\ref{fig:board and governor A}. Recall that if Beach is left in prison, then in outcome $w_1$, the Board is responsible for seeing to it; in outcome $w_2$, the Governor is responsible for seeing to it and also counterfactually responsible for it. If Beach is set free in outcome $w_3$, then the Governor is counterfactually responsible for it. Thus, for the statements ``Beach is left in prison'' and ``Beach is set free'', there is no responsibility gap in this game.

\subsection{In extensive form games with two agents}

In Theorem~\ref{no gap 2 agents theorem} below, I show that the two forms of responsibility discussed in this paper leave no gap in any extensive form games with only two agents. Let us start, however, by formally defining the {\bf\em gap formulae} $\Gc(\phi)$ and $\Gs(\phi)$ for any formula $\phi\in\Phi$. Informally, the formula $\Gc(\phi)$ means that $\phi$ is true and nobody is counterfactually responsible for it. The formula $\Gs(\phi)$ says the same for the seeing-to-it form of responsibility.
\begin{align}
    \Gc(\phi)&:=\phi\wedge \bigwedge_{a\in\mathcal{A}}\neg\C_a\phi,\label{eq:Gapc definition}\\
    \Gs(\phi)&:=\phi\wedge \bigwedge_{a\in\mathcal{A}}\neg\S_a\phi.\label{eq:Gaps definition}
\end{align}
I define the combined responsibility gap formula $\Gcs(\phi)$ as the conjunction $\Gc(\phi)\wedge\Gs(\phi)$.
%
%
%
%
%
%
%
%
%
%
The proof of Theorem~\ref{no gap 2 agents theorem} uses the following well-known lemma. To keep this paper self-contained, I prove the lemma in Appendix~\ref{app_sec:proof lm win in two agent}.

\begin{lemma}\label{lm:win in two agent}
For any formula $\phi\in\Phi$ and any node $n$ in a two-agent extensive form game between agents $a$ and $b$, if $n\notin win_a(\[\phi\])$, then $n\in win_b(\[\neg\phi\])$.
\end{lemma}




\begin{theorem}\label{no gap 2 agents theorem}
For any formula $\phi\in\Phi$ and any two-agent extensive form game $G$, if $\[\phi\]\neq\Omega(G)$, then $\[\Gcs(\phi)\]=\varnothing$. 
\end{theorem}
\begin{proof}
I prove this theorem by showing that, for any outcome $w\in\Omega(G)$, if $w\in\[\Gs(\phi)\]$, then $w\notin\[\Gc(\phi)\]$. 
By statement~\eqref{eq:Gapc definition} and items~2 and 3 of Definition~\ref{sat}, it suffices to show the existence of an agent $a$ such that $w\in\[\C_a\phi\]$.

By statement~\eqref{eq:Gaps definition} and items~2 and 3 of Definition~\ref{sat}, the assumption $w\in\[\Gs(\phi)\]$ implies that
\begin{equation}\label{18-apr-f}
w\in\[\phi\]
\end{equation}
and
\begin{equation}\label{18-apr-e}
w\notin\[\S_b\phi\]
\end{equation}
for each agent $b\in\mathcal{A}$.
At the same time, by the assumption $\[\phi\]\neq\Omega(G)$, statement~\eqref{18-apr-f}, and Lemma \ref{lm:achievement point existence}, there exists an agent $b$ and a $\[\phi\]$-achievement point $n$ by agent $b$ such that $w\preceq n$.  
Hence, by item~5 of Definition~\ref{sat} and statement~\eqref{18-apr-e}, there is a node $m$ such that $w\preceq m$ and $m\notin win_b(\[\phi\])$. Since $G$ is a two-agent game, let $a$ be the agent in the game distinct from agent $b$. Then, $m\in win_a(\[\neg\phi\])$ by Lemma~\ref{lm:win in two agent}. Hence, $w\in\[\C_a\phi\]$ by item~4 of Definition \ref{sat} because of statement~\eqref{18-apr-f} and that $w\preceq m$.
\end{proof}

\subsection{In extensive form games with more agents}

To see if there is a responsibility gap in extensive form games with {\em more than} two agents, let us go back to the story in the introduction, which is not as simple as I tried to make it. In over 30 years that separate Kim Nees's murder and Beach's attorney filing an executive clemency petition, the case became highly controversial in Montana due to the lack of direct evidence and doubts about the integrity of the interrogators. By the time the petition was filed, the Board had already made clear its intention to deny the petition, while the Governor expressed his support for the clemency~\cite{b15mt}. 

Then, something very unusual happened.
On 4 December 2014, a bill was introduced in the Montana House of Representatives that would allow the Governor to grant executive clemency no matter what the decision of the Board is. This bill aimed to strip the Board from the power that it had from the day the State of Montana was founded in 1889. Although the bill would affect the Governor's power to grant clemency in other cases as well, the primary goal of the legislation was to give the Governor a chance to free Beach~\cite{mip2023never}.

\begin{figure}[ht]
\begin{center}
\vspace{0mm}
\scalebox{0.5}{\includegraphics{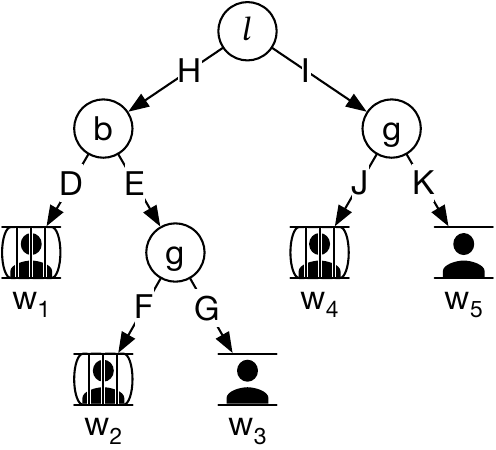}}
\caption{Barry Beach's case of clemency}
\label{fig:house, board and governor}
\end{center}
\end{figure}
Figure~\ref{fig:house, board and governor} depicts the extensive form game that captures the situation after the bill was introduced. If the Montana State Legislature rejects (H) the bill, then the game continues as in Figure~\ref{fig:board and governor A}. If the Legislature approves (I) the bill, then the Governor unilaterally decides whether to grant the clemency. 
In this new three-agent game, the Governor is responsible for Beach being left in prison in outcomes $w_2$ and $w_4$ both counterfactually and for seeing to it. The Governor is also counterfactually responsible for Beach being freed in outcomes $w_3$ and $w_5$. However, in outcome $w_1$, nobody is responsible for the fact that Beach is left in prison either counterfactually or for seeing to it. In particular, the Board is not responsible for seeing to this because it no longer has an upfront ability to guarantee that Beach is left in prison in the outcome. Therefore, by statements~\eqref{eq:Gapc definition} and \eqref{eq:Gaps definition},
\begin{equation}\label{9-may-a}
\[\Gcs(\text{``Beach is left in prison''})\]=\{w_1\}.    
\end{equation}
This example shows that the responsibility gap may exist in extensive form games with more than two agents. In other words, the two forms of responsibility discussed here are not enough to have a responsible agent in every situation. 

\subsection{Hierarchy of responsibility gaps}

A further question about the responsibility gap is if there is an agent responsible for the gap. The responsibility for the gap, or {\em the responsibility for the lack of a responsible agent}, is a natural concept that applies to many real-world situations. For instance, the managers who assign tasks and the governing bodies that set the rules are often responsible for the lack of a responsible person. In the example in Figure~\ref{fig:house, board and governor}, it is the Legislature that is counterfactually responsible for the gap in outcome $w_1$. Indeed, the Legislature could prevent the formula $\Gcs(\text{``Beach is left in prison''})$ from being true by approving (I) the bill:
$$w_1\in\[\C_l\Gcs(\text{``Beach is left in prison''})\].$$
In addition, in my example, the Board is also counterfactually responsible for the gap in outcome $w_1$.

I also consider the lack of responsibility for the gap. By {\em second-order gap} for a formula $\phi$ I mean the presence of outcomes in which $\Gcs(\phi)$ is true and nobody is responsible for it. In a real-world situation, the first-order responsibility gap often shows that the managers do not assign tasks in an accountable way, while the second-order responsibility gap is often caused by a failure of the leadership to properly define the roles of the managers so that the managers had no way to assign tasks in an accountable way.

In general, for an arbitrary formula $\phi\in\Phi$ and any integer $i\ge 0$, I consider the $i^\text{th}$-order gap statement $\Gcs_i(\phi)$ defined recursively as:
\begin{equation}\label{Gcs definition}
\Gcs_i(\phi) := \begin{cases}
\Gcs_{i-1}(\phi) \wedge
\bigwedge_{a\in\mathcal{A}}\neg\C_a\Gcs_{i-1}(\phi)\\ \hspace{1.32cm}  \wedge \bigwedge_{a\in\mathcal{A}} \neg\S_a\Gcs_{i-1}(\phi), & i\geq 1;\\
\phi,   & i=0.
\end{cases}
\end{equation}
\noindent
I also consider the $i^\text{th}$-order counterfactual gap statement $\Gc_i(\phi)$ defined recursivedly as:
\begin{equation}\label{Gc definition}
\Gc_i(\phi) := \begin{cases}
\Gc_{i-1}(\phi)\wedge\bigwedge_{a\in\mathcal{A}}\neg\C_a\Gc_{i-1}(\phi), & i\geq 1;\\
\phi, & i=0.
\end{cases}
\end{equation} 
One can similarly define the $i^\text{th}$-order seeing-to-it gap statement $\Gs_i(\phi)$. It is easy to see from statements~\eqref{eq:Gapc definition} and \eqref{eq:Gaps definition} that the first order gap statements $\Gcs_1(\phi)$, $\Gc_1(\phi)$, and $\Gs_1(\phi)$ are equivalent to the previously discussed gap statement $\Gcs(\phi)$, $\Gc(\phi)$, and $\Gs(\phi)$, respectively.

As shown in Theorem~\ref{no gap 2 agents theorem}, in two-agent extensive form games, the truth set $\[\Gcs(\phi)\]$ is empty for each formula $\phi\in\Phi$ such that $\[\phi\]\neq\Omega(G)$. Informally, this means that there is no responsibility gap in two-agent extensive form games. At the same time, the example depicted in Figure~\ref{fig:house, board and governor} shows that such a gap might exist in games with more than two agents. This observation is correct.
In Appendix~\ref{app_sec:gap instance}, for each integer $i\ge 2$, I construct an extensive form game in which the truth set $\[\Gcs_i(\phi)\]$ is {\em not} empty.

In spite of this, in Theorem~\ref{th:no higher-order Gc} and Corollary~\ref{cr:no higher-order Gcs gap} below, I show that, for any extensive form game, the sets $\[\Gc_i(\phi)\]$ and $\[\Gcs_i(\phi)\]$ are empty for large enough integer $i$. Informally, the higher-order responsibility gap does not exist in any extensive form game if I consider sufficiently high order. 

Let us first show two lemmas that are used later to prove Theorem~\ref{th:no higher-order Gc}. These lemmas show that the set $\[\Gc_{i}(\phi)\]$ monotonously shrinks to empty as the order $i$ increases.

\begin{lemma}\label{23-apr-a}
$\[\Gc_{i+1}(\phi)\]\subseteq\[\Gc_{i}(\phi)\]$
for any formula $\phi\in\Phi$ and any integer $i\ge 0$. 
\end{lemma}
\begin{proof}
The statement of the lemma follows from statement~\eqref{Gc definition} and item 3 of Definition~\ref{sat}.
\end{proof}






\begin{lemma}\label{23-apr-b}
For any formula $\phi\in\Phi$, any integer $i\ge 0$, and any extensive form game $G$, if
$\varnothing\subsetneq\[\Gc_{i}(\phi)\]\subsetneq \Omega(G)$,
then
$\[\Gc_{i+1}(\phi)\]\subsetneq\[\Gc_{i}(\phi)\]$.
\end{lemma}

\begin{proof}
The assumption $\varnothing\subsetneq\[\Gc_{i}(\phi)\]\subsetneq\Omega(G)$, by item~2 of Definition~\ref{sat}, implies that $\varnothing\subsetneq\neg\[\Gc_{i}(\phi)\]\subsetneq\Omega(G)$. Then, on the one hand, there is an outcome $w\in \[\neg\Gc_{i}(\phi)\]$.
On the other hand, by Lemma~\ref{lm:achievement point existence}, there is an $\[\neg\Gc_{i}(\phi)\]$-achievement point $n$ by an agent $a$ such that $w\preceq n$. Thus, by Definition~\ref{df:achievement point}, 
\begin{enumerate}
    \item $parent(n)$ is labelled with agent $a$;
    \item there exists an outcome $w'$ such that $w'\preceq parent(n)$ and  $w'\notin \[\neg\Gc_{i}(\phi)\]$;
    \item $w''\in \[\neg\Gc_{i}(\phi)\]$ for each outcome $w''$ such that $w''\preceq n$.
\end{enumerate}
Item~3 above implies that
$n\in win_a(\[\neg\Gc_{i}(\phi)\])$ 
by Definition~\ref{df:winning set}.
Hence, by item~2 of Definition~\ref{df:winning set} and item~1 above,
\begin{equation}\label{eq:30-July-9}
parent(n)\in win_a(\[\neg\Gc_{i}(\phi)\]).
\end{equation}
By the part $w'\notin \[\neg\Gc_{i}(\phi)\]$ of item~2 above and item~2 of Definition~\ref{sat},
\begin{equation}\label{eq:30-July-10}
w'\in\[\Gc_{i}(\phi)\].
\end{equation}
Thus, $w'\in \[\C_a\Gc_{i}(\phi)\]$ by the part $w'\preceq parent(n)$ of item~2 above, statement~\eqref{eq:30-July-9}, and item~4 of Definition~\ref{sat}.
Then, $w'\notin\[\neg\C_a\Gc_{i}(\phi)\]$ by item~2 of Definition~\ref{sat}. 
Hence, $w'\notin\[\Gc_{i+1}(\phi)\]$ by statement~\eqref{Gc definition} and item~3 of Definition~\ref{sat}.
Then, $\[\Gc_{i+1}(\phi)\]\neq\[\Gc_{i}(\phi)\]$ by statement~\eqref{eq:30-July-10}.
Therefore, $\[\Gc_{i+1}(\phi)\]\subsetneq\[\Gc_{i}(\phi)\]$ by Lemma~\ref{23-apr-a}.
\end{proof}

\begin{theorem}\label{th:no higher-order Gc}
$\[\Gc_i(\phi)\]=\varnothing$ for each integer $i\ge |\Omega(G)|-1$ and each formula $\phi\in\Phi$ such that $\[\phi\]\subsetneq\Omega(G)$.
\end{theorem}
\begin{proof}
By the assumption $\[\phi\]\subsetneq\Omega(G)$ of this theorem, statement~\eqref{Gc definition}, and Lemma~\ref{23-apr-a},
$$
\Omega(G)\supsetneq \[\phi\]= \[\Gc_0(\phi)\]\supseteq \[\Gc_1(\phi)\]\supseteq \[\Gc_2(\phi)\]\supseteq \dots 
$$
Note that $|\[\phi\]|\leq|\Omega(G)|-1$ by the assumption $\[\phi\]\subsetneq\Omega(G)$ of this theorem. Therefore, $\[\Gc_i(\phi)\]=\varnothing$ for each integer $i\ge |\Omega(G)|-1$ by Lemma~\ref{23-apr-b}.
\end{proof}

The next corollary follows from the above theorem and the observation that $\[\Gcs_i(\phi)\]\subseteq\[\Gc_i(\phi)\]$. I give the formal proof in Appendix~\ref{app_sec:corollary no Gcs gap}.

\begin{corollary}\label{cr:no higher-order Gcs gap}
$\[\Gcs_i(\phi)\]=\varnothing$ for each integer $i\ge |\Omega(G)|-1$ and each formula $\phi\in\Phi$ such that $\[\phi\]\subsetneq\Omega(G)$.
\end{corollary}

\section{Conclusion}

The existing definitions of seeing-to-it modalities have clear shortcomings when viewed as possible forms of responsibility. In this paper, I combined them into a single definition of seeing-to-it responsibility that addresses the shortcomings. By proving the undefinability results, I have shown that the proposed notion is semantically independent of the counterfactual responsibility already discussed in the literature. The other important contribution of this work is the hierarchy of responsibility gaps. I believe that taking into account higher-order responsibilities is an important step towards designing better mechanisms in terms of responsibility attribution. In the future, I would like to study how the gap results could be extended to the setting of games with imperfect information, where even Lemma~\ref{lm:win in two agent} does not hold.

One more thing, if you are curious about the ending of Beach's story, 
in January 2015, the Montana House of Representatives approved the bill that changes the clemency procedure. By doing so, they, perhaps unintentionally, prevented the potential responsibility gap existing in outcome $w_1$ of Figure~\ref{fig:house, board and governor}. 
In November of the same year, the Governor granted clemency to Beach~\cite{b15mt}.

\clearpage

\bibliography{naumov,this}

\clearpage

\appendix

\begin{center}
\bf \Large   Technical Appendix

\vspace{2mm}
\end{center}

\section{Mutual undefinability between $\C$ and $\S$}\label{app_sec:undefinability}

I use a technique named ``truth set algebra''~\cite{kn22arxiv} that is different from the traditional ``bisimulation'' method. Unlike the ``bisimulation'' method, the ``truth sets algebra'' technique uses a single model. Grossly speaking, I define an extensive form game, use it to show the semantic inequivalence between formula $\C_ap$ and any formula in language $\Phi$ that does not use modality $\C$, and do the same for modality $\S$.

\subsection{Undefinability of $\C$ via $\S$}\label{app_sec:undefinability C via S}

In this subsection, I consider an extensive form game between agents $a$, $b$, and $c$ depicted in the top of Figure~\ref{C via S game figure}. It has four outcomes: $w_1$, $w_2$, $w_3$, and $w_4$. Without loss of generality\footnote{Alternatively, additional agents and propositional variables can be assumed to be present but not used as labels. In particular, according to items~4 and 5 of Definition~\ref{sat}, it can be deduced that $\[\S_d\phi\]=\[\C_d\phi\]=\varnothing=\[\bot\]$ for any agent $d$ which is {\em not} used as a label in the game and any formula $\phi\in\Phi$.}, I assume that the language contains only agents $a$, $b$, $c$ and a single propositional variable $p$. Outcomes $w_1$ and $w_3$ are labelled with the set $\{p\}$ and outcomes $w_2$ and $w_4$ are labelled with the empty set, see the top of Figure~\ref{C via S game figure}.

\begin{figure}[ht]
\begin{center}
\vspace{0mm}
\scalebox{0.47}{\includegraphics{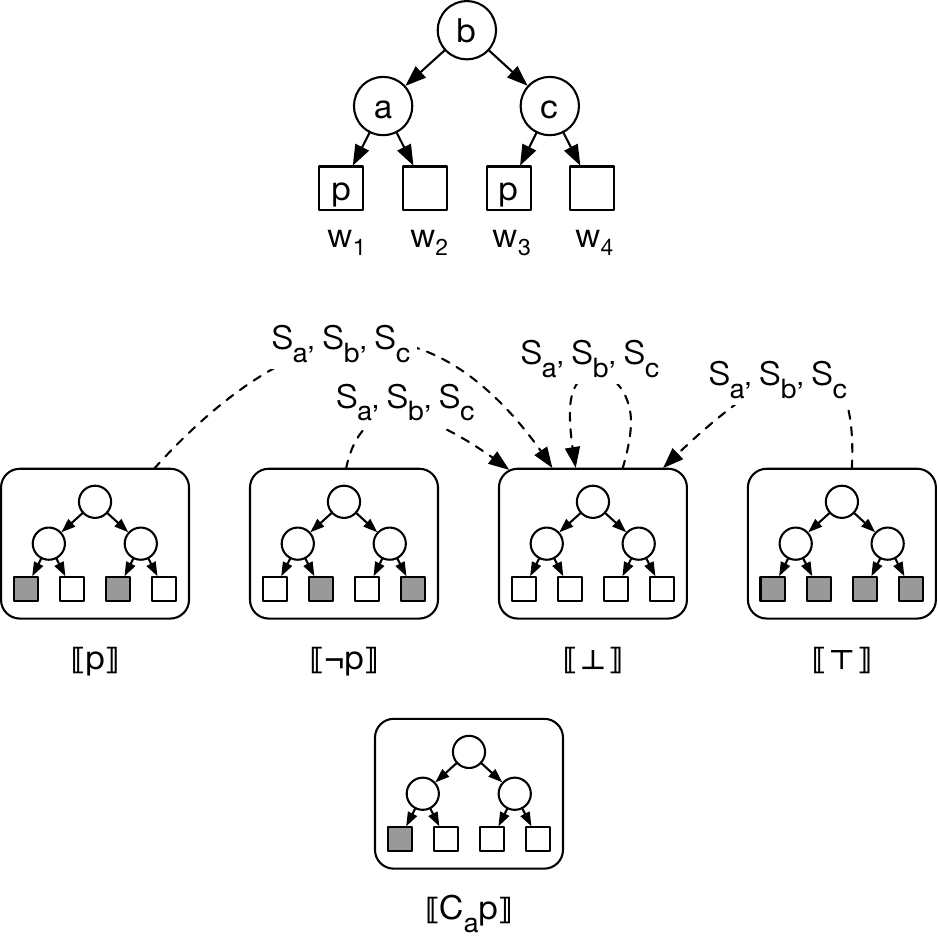}}
\caption{Towards the proof of undefinability of $\C$ via $\S$}\label{C via S game figure}
\end{center}
\end{figure}

I visualise the truth set $\[\phi\]$ of a formula $\phi\in\Phi$ by {\em shading grey} the outcomes of the game that belong to the set $\[\phi\]$. To be specific, I consider a family of truth sets $\mathcal{F}=\{\[p\],\[\neg p\],\[\bot\],\[\top\]\}$ for the above game and visualise these truth sets with the four diagrams in the middle row of Figure~\ref{C via S game figure}.

\begin{lemma}\label{13-oct-a}
$\[\S_g\phi\]=\[\bot\]$ for any agent $g\in\{a,b,c\}$ and any formula $\phi\in\Phi$ such that $\[\phi\]\in\mathcal{F}$. 
\end{lemma}
\begin{proof}
I first show that if $\[\phi\]=\[p\]$, then $\[\S_g\phi\]=\[\bot\]$ for each agent $g\in\{a,b,c\}$.

Indeed, $\[\phi\]=\[p\]=\{w_1,w_3\}$. Then, by Definition~\ref{df:winning set}, the root node of the tree does {\em not} belong to the set $win_g(\[\phi\])$ for each agent $g\in\{a,b,c\}$. Thus, for any agent $g\in \{a,b,c\}$, there is no single path from the root to an outcome such that all nodes of this path belong to the set $win_g(\[\phi\])$. Hence, by item 5(a) of Definition~\ref{sat}, none of the agents is responsible for seeing to $\phi$ in any of the outcomes. Therefore, for each formula $\phi\in\Phi$ and each agent $g\in\{a,b,c\}$, if $\[\phi\]=\[p\]$, then $\[\S_g\phi\]=\[\bot\]$. This is captured by the dashed arrow from the diagram for the set $\[p\]$ to the diagram for the set $\[\bot\]$ in the middle row of Figure~\ref{C via S game figure}.

The justifications for the cases where $\[\phi\]=\[\neg p\]$ and $\[\phi\]=\[\bot\]$ are similar.

In the case where $\[\phi\]=\[\top\]$, observe that the set $\[\top\]$ is the set of all outcomes in the game. Thus, there is no $\[\phi\]$-achievement point by Definition~\ref{df:achievement point}. Hence, by item~5(b) of Definition~\ref{sat}, none of the agents is responsible for seeing to $\phi$ in any of the outcomes. Then, for each formula $\phi\in\Phi$ and each agent $g\in\{a,b,c\}$, if $\[\phi\]=\[\top\]$, then $\[\S_g\phi\]=\[\bot\]$.  This is captured by the dashed arrow from the diagram for the set $\[\top\]$ to the diagram for the set $\[\bot\]$.
\end{proof}

\begin{lemma}\label{10-nov-a}
$\[\phi\]\in\mathcal{F}$ for any formula $\phi$ that only contains modalities $\S_a$, $\S_b$, and $\S_c$.
\end{lemma}
\begin{proof}
I prove the statement of the lemma by induction on the structural complexity of formula $\phi$. 

If $\phi$ is propositional variable $p$, then the statement of the lemma is true because $\[p\]$ is an element of the set $\mathcal{F}$.    

If formula $\phi$ has the form $\neg\psi$, where $\[\psi\]\in\mathcal{F}$ by the induction hypothesis, then, by item~2 of Definition~\ref{sat}, the set of outcomes $\[\phi\]$ is the complement of the set $\[\psi\]$. In other words, the diagram for the set $\[\phi\]$ is obtained from the diagram for the set $\[\psi\]$ by swapping the white and grey colours of the squares. Observe in the middle row of Figure~\ref{C via S game figure} that such a swap for any of the sets in $\mathcal{F}$ is again one of the sets in $\mathcal{F}$. Therefore, $\[\phi\]\in\mathcal{F}$.

If formula $\phi$ has the form $\psi_1\wedge\psi_2$, where $\[\psi_1\],\[\psi_2\]\in\mathcal{F}$ by the induction hypothesis, then, by item~3 of Definition~\ref{sat}, the set $\[\phi\]$ is the intersection of the sets $\[\psi_1\]$ and $\[\psi_2\]$. Observe in the middle row of Figure~\ref{C via S game figure} that the intersection of any two of the sets in $\mathcal{F}$ is again one of the sets in $\mathcal{F}$. For example, $\[p\]\cap \[\neg p\]= \[\bot\]$. Therefore, $\[\phi\]\in\mathcal{F}$.

If formula $\phi$ has the form $\S_g\psi$, where $g\in\{a,b,c\}$, then $\[\psi\]\in\mathcal{F}$ by the induction hypothesis. In this case, the statement of the lemma follows from Lemma~\ref{13-oct-a} and that $\[\bot\]\in\mathcal{F}$.
\end{proof}

\begin{lemma}
    $\[\C_a p\]\notin\mathcal{F}$.
\end{lemma}
\begin{proof}
    The diagram for the truth set $\[\C_a p\]$ is depicted in the bottom of Figure~\ref{C via S game figure}. Indeed, $\[p\]=\{w_1,w_3\}$. On the path to outcome $w_1$, the node labelled with agent $a$ belongs to the set $win_a(\[p\])$ by Definition~\ref{df:winning set}. This means agent $a$ has a strategy (``go right'') to prevent $p$. However, agent $a$ has no such a strategy on the path to outcome $w_3$. Therefore, $\[\C_a p\]=\{w_1\}$ by item~4 of Definition~\ref{sat}. 
\end{proof} 

\begin{figure*}[ht]
\begin{center}
\vspace{0mm}
\scalebox{0.47}{\includegraphics{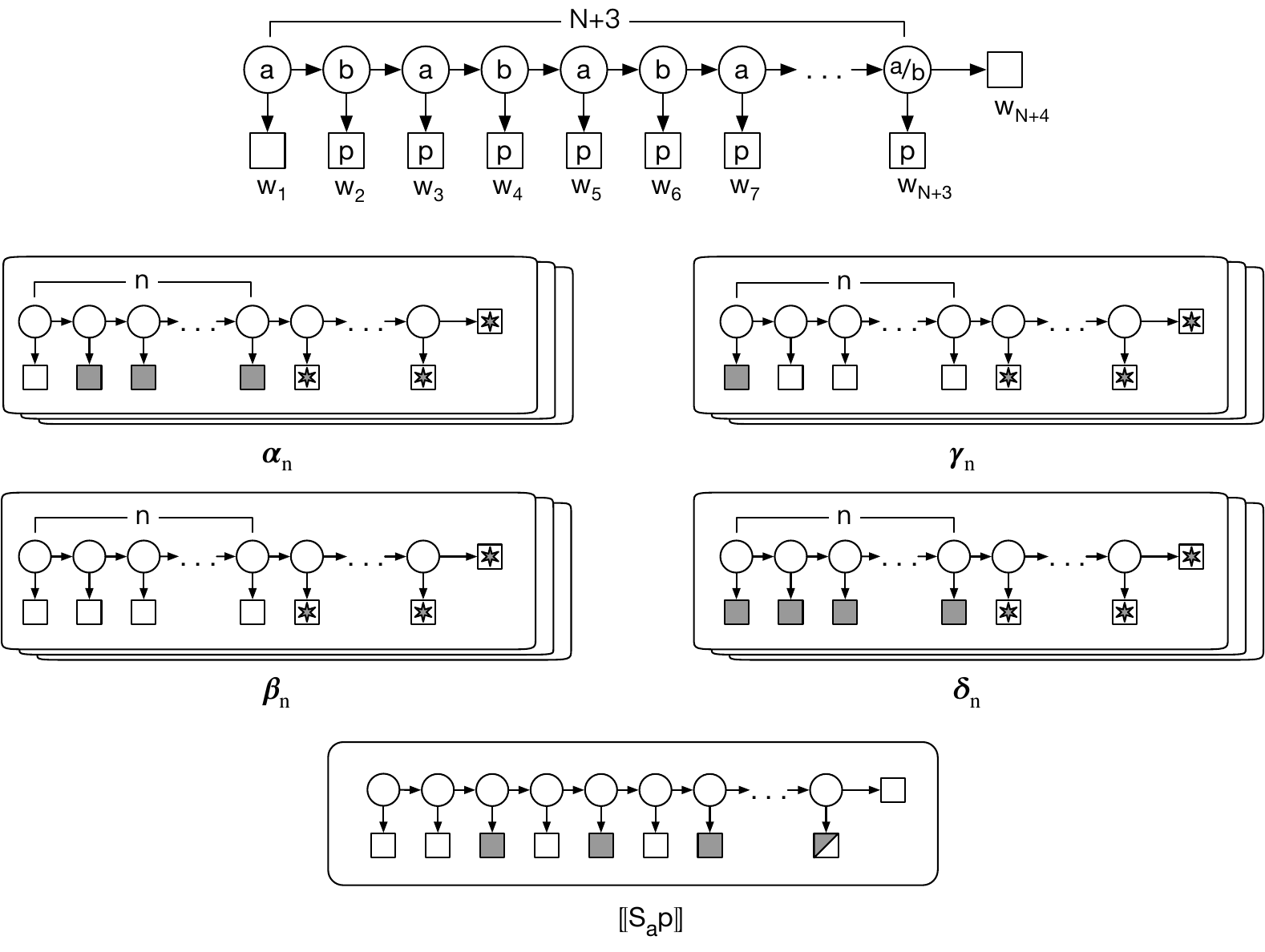}}
\caption{Towards the proof of undefinability of $\S$ via $\C$}\label{S via C game figure}
\end{center}
\end{figure*}

\vspace{1mm}
Theorem~\ref{th:undefinability C via S} follows from Definition~\ref{semantically equivalent} and the two previous lemmas.

\noindent\textbf{Theorem~\ref{th:undefinability C via S} (undefinability of $\C$ via $\S$)}
\textit{The formula $\C_a p$ is not semantically equivalent to any formula in language $\Phi$ that does not contain modality $\C$.}

\subsection{Undefinability of $\S$ via $\C$}\label{app_sec:undefinability S via C}

To show that a formula $\phi$ is not semantically equivalent to any formula in a language $\Psi$, by Definition~\ref{semantically equivalent}, it suffices for each formula $\psi\in\Psi$ to construct a game (model) $G_\psi$ such that $\[\phi\]\neq \[\psi\]$ in game $G_\psi$. In Section~\ref{app_sec:undefinability C via S}, I have been able to construct a {\em uniform} game that does not depend on formula $\psi$. This made the whole proof relatively simple.  I do not know how to construct such a uniform game for my second undefinability result presented in this subsection. The proof here constructs a different game $G_\psi$ for each formula $\psi\in\Psi$. More precisely, I construct a game $G_N$ where $N$ is the number of occurrences of modality $\C$ in formula $\psi$. I formally describe the construction below.

Without loss of generality, in this subsection, I assume that the language has a single propositional variable $p$ and two agents, $a$ and $b$. I consider a game $G_N$ with a parameter $N$, where $N$ is an arbitrary positive integer. Game $G_N$ is depicted at the top of Figure~\ref{S via C game figure}. It has $N+3$ non-leaf nodes and $N+4$ outcomes (leaf nodes): $w_1$, \dots, $w_{N+4}$. The non-leaf nodes are labelled with agents $a$ and $b$, who take turns to decide whether to terminate the game (by going down) or to continue (by going to the right). The game terminates after at most $N+3$ turns. Which agent makes the last move depends on the parity of $N$. Outcomes $w_1$ and $w_{N+4}$ are labelled with the empty set while  outcomes $w_2$, \dots, $w_{N+3}$ are labelled with the set $\{p\}$.

I consider four {\em families} of truth sets: $\alpha_n$, $\beta_n$, $\gamma_n$, and $\delta_n$ for each integer $n$ such that $3\le n\le N+3$.
The family $\alpha_n$ of truth sets consists of all subsets of the set $\{w_1,\dots,w_{N+4}\}$ that exclude outcome $w_1$ and include outcomes $w_2,\dots,w_n$:
\begin{equation*}
\alpha_n:=\big\{\{w_2,\dots,w_n\}\cup X\,|\, X\subseteq \{w_{n+1},\dots,w_{N+4}\}\big\}.
\end{equation*}
I similarly define other families of truth sets: 
\begin{align*}
&\beta_n:=\big\{X\,|\, X\subseteq \{w_{n+1},\dots,w_{N+4}\}\big\};\\
&\gamma_n:=\big\{\{w_1\}\cup X\,|\, X\subseteq \{w_{n+1},\dots,w_{N+4}\}\big\};\\
&\delta_n:=\big\{\{w_1, w_2,\dots,w_n\}\cup X\,|\, X\subseteq \{w_{n+1},\dots,w_{N+4}\}\big\}.
\end{align*}
I visualise families $\alpha_n$, $\beta_n$, $\gamma_n$, and $\delta_n$ in the middle two rows of Figure~\ref{S via C game figure}. In this figure, the asterisk $*$ is used as the {\em wildcard} to mark the outcomes that {\em might but do not have to} belong to a set in the corresponding family. 

\begin{lemma}\label{15-oct-a}
For any formulae $\phi,\psi\in \Phi$ and any $n\ge 0$, if $\[\phi\],\[\psi\]\in\alpha_n\cup\beta_n\cup\gamma_n\cup\delta_n$, then $$\[\neg\phi\],\[\phi\wedge\psi\]\in\alpha_n\cup\beta_n\cup\gamma_n\cup\delta_n.$$
\end{lemma}
\begin{proof}
Observe that, by the definition of families $\alpha_n$, $\beta_n$, $\gamma_n$, and $\delta_n$, the family of sets $\alpha_n\cup\beta_n\cup\gamma_n\cup\delta_n$ is close respect to negation and conjunction. Then, the statement of this lemma follows from items~2 and 3 of Definition~\ref{sat}.
\end{proof}

\begin{lemma}\label{16-oct-a}
For any integer $n\ge 3$ and any formula $\phi\in\Phi$,
\begin{enumerate}
    \item if $\[\phi\]\in \alpha_n$, then $\[\C_a\phi\]\in \alpha_n$ and $\[\C_b\phi\]\in \beta_{n-1}$;
   
    \item if $\[\phi\]\in \beta_n$, then $\[\C_a\phi\]\in \beta_n$ and $\[\C_b\phi\]\in \beta_n$;
    
    \item if $\[\phi\]\in \gamma_n$, then $\[\C_a\phi\]\in \gamma_n$ and $\[\C_b\phi\]\in \beta_n$;
    
    \item if $\[\phi\]\in \delta_n$, then $\[\C_a\phi\]\in \beta_{n-1}$ and $\[\C_b\phi\]\in \beta_{n-1}$.
\end{enumerate}
\end{lemma}
\begin{proof}
Suppose that $\[\phi\]\in \alpha_n$ for some integer $n\ge 3$. Then, 
\begin{equation}\label{eq:26-July-1}
w_1\notin \[\phi\],
\end{equation}
\begin{equation}\label{eq:26-July-2}
w_2,\dots,w_n\in \[\phi\].
\end{equation}
Statement~\eqref{eq:26-July-1} implies
\begin{equation}\label{eq:26-July-3}
w_1\notin\[\C_a\phi\]
\end{equation}
by item~4 of Definition~\ref{sat} and
$
w_1\in\[\neg\phi\]
$
by item~2 of Definition~\ref{sat}. Hence, the root node, which is labelled by agent $a$, belongs to the set $win_a(\[\neg\phi\])$ by Definition~\ref{df:winning set}. Then,
\begin{equation}\label{eq:26-July-4}
w_2,\dots,w_n\in \[\C_a\phi\]
\end{equation}
by statement~\eqref{eq:26-July-2} and item~4 of Definition~\ref{sat}. Therefore, $\[\C_a\phi\]\in \alpha_n$ by statements~\eqref{eq:26-July-3} and \eqref{eq:26-July-4} and the definition of family $\alpha_n$. 

\noindent Note that, by statement~\eqref{eq:26-July-1} and item~4 of Definition~\ref{sat},
\begin{equation}\label{eq:26-July-5}
w_1\notin \[\C_b\phi\].
\end{equation}
Also, observe that 
none of the non-leaf nodes above outcomes $w_1,\dots,w_{n-1}$ belongs to the set $win_b(\[\neg\phi\])$ by Definition~\ref{df:winning set} because of statements~\eqref{eq:26-July-1} and \eqref{eq:26-July-2} and that the root node above $w_1$ is labelled by agent $a$. Hence,
\begin{equation}\label{eq:26-July-6}
w_2,\dots,w_{n-1}\notin\[\C_b\phi\]
\end{equation}
by item~4 of Definition~\ref{sat}. Therefore,
$\[\C_b\phi\]\in\beta_{n-1}$
by statements~\eqref{eq:26-July-5} and \eqref{eq:26-July-6} and the definition of family $\beta_n$.


The proofs of the other three parts of the lemma are similar.
\end{proof}

\begin{lemma}\label{16-oct-b}
For any formula $\phi\in\Phi$, any agent $g\in\{a,b\}$, and any integer $n\ge 3$, if $\[\phi\]\in \alpha_n\cup \beta_n\cup \gamma_n \cup \delta_n$, then  $\[\C_g\phi\]\in \alpha_{n-1}\cup \beta_{n-1}\cup \gamma_{n-1} \cup \delta_{n-1}$.  
\end{lemma}
\begin{proof}
The statement of the lemma follows from Lemma~\ref{16-oct-a} and the facts that 
$\alpha_n\subseteq\alpha_{n-1}$, 
$\beta_n\subseteq\beta_{n-1}$, 
$\gamma_n\subseteq\gamma_{n-1}$, and
$\delta_n\subseteq\delta_{n-1}$ for each $n\ge 1$.
\end{proof}

\begin{lemma}\label{16-oct-c}
For any formula $\phi$ that does not contain modality $\S$ and any integer $k\le N$, if formula $\phi$ contains at most $k$ occurrences of modality $\C$, then, 
\begin{equation}
    \[\phi\]\in \alpha_{N+3-k}\cup \beta_{N+3-k}\cup \gamma_{N+3-k} \cup \delta_{N+3-k}.\notag
\end{equation}
\end{lemma}
\begin{proof}
I prove the statement of the lemma by structural induction on formula $\phi$. If $\phi$ is propositional variable $p$, then $\[\phi\]=\[p\]=\{w_2,\dots,w_{N+3}\}\in\alpha_{N+3}\subseteq\alpha_{N+3-k}$. 

If formula $\phi$ is a disjunction or a negation, then the statement of the lemma follows from the induction hypothesis by Lemma~\ref{15-oct-a}.  

If formula $\phi$ has the form $\C_g\psi$, where $g\in\{a,b\}$, then formula $\psi$ contains at most $k-1$ occurrences of modality $\C$. Then, the statement of the lemma follows from the induction hypothesis by Lemma~\ref{16-oct-b}.
\end{proof}

\vspace{1mm}
The truth set $\[\S_a p\]$ is shown at the bottom of Figure~\ref{S via C game figure}. However, to prove the undefinability result, I only need to observe the following lemma:
\begin{lemma}\label{16-oct-d}
$w_2\notin \[\S_a p\]$ and $w_3\in \[\S_a p\]$.
\end{lemma}
\begin{proof}
Observe that, in game $G$ depicted in the top of Figure~\ref{S via C game figure}, node $w_2$ is the $\[p\]$-achievement point by agent $b$ on the path of play toward outcome $w_2$. Then, by Lemma~\ref{lm:achievement point existence}, there is no $\[p\]$-achievement point by agent $a$ on the path of play toward outcome $w_2$. Hence, $w_2\notin \[\S_a p\]$ by item~5(b) of Definition~\ref{sat}.

The non-leaf nodes above outcomes $w_1,w_2,w_3$ belong to the set $win_a(\[p\])$ by Definition~\ref{df:winning set}. At the same time, node $w_3$ is the $\[p\]$-achievement point by agent $a$ on the path of play toward outcome $w_3$. Hence, $w_3\in\[\S_a p\]$ by item~5 of Definition~\ref{sat}.
\end{proof}

\vspace{1.5mm}
\noindent\textbf{Theorem~\ref{th:undefinability S via C} (undefinability of $\S$ via $\C$)}
\textit{The formula $\S_a p$ is not semantically equivalent to any formula in language $\Phi$ that does not contain modality $\S$.}

\begin{proof} Assume the opposite. Then, by Definition~\ref{semantically equivalent}, there is a formula $\phi\in\Phi$ not containing modality $\S$ such that $\[\S_a p\]=\[\phi\]$ in any extensive form game. Let $N$ be the number of occurrences of modality $\C$ in formula $\phi$. Then, by Lemma~\ref{16-oct-c}, in game $G_N$,
\begin{equation}
\begin{aligned}
    \[\S_a p\]\!=\!\[\phi\]&\in \alpha_{N+3-N}\cup \beta_{N+3-N}\cup \gamma_{N+3-N} \cup \delta_{N+3-N}\\
    &=\alpha_{3}\cup \beta_{3}\cup \gamma_{3} \cup \delta_{3}. \notag
\end{aligned}
\end{equation}
However, $\[\S_a p\]\notin \alpha_{3}\cup \beta_{3}\cup \gamma_{3} \cup \delta_{3}$ by Lemma~\ref{16-oct-d} and the definition of families $\alpha_3$, $\beta_3$, $\gamma_3$, and $\delta_3$.
\end{proof}

\section{Properties of higher-order responsibility}\label{app_sec:properties of higher order responsibility}

\subsection{Semantic equivalence of $\C_a\C_a\phi$ and $\C_a\phi$}\label{app_sec:CC=C}

The main result of this subsection is in Property~\ref{pp:CC=C}. I first show the next lemma to increase the readability of its proof.

\begin{lemma}\label{lm:26-apr-a}
For any formulae $\phi,\psi\in\Phi$ and each agent $a$, if $\[\phi\]\subseteq\[\psi\]$, then $win_a(\[\phi\])\subseteq win_a(\[\psi\])$.
\end{lemma}
\begin{proof}
I prove this lemma by backward induction on node $n$. In other words, by induction on the depth of the subtree rooted at node $n$. I show that if $n\in win_a(\[\phi\])$, then $n\in win_a(\[\psi\])$.

If node $n$ is a leaf, then the assumption $n\in win_a(\[\phi\])$ implies $n\in\[\phi\]$ by the requirement of ``minimal set'' in Definition~\ref{df:winning set}. Then, $n\in\[\psi\]$ by the assumption $\[\phi\]\subseteq\[\psi\]$ of this lemma. Therefore, $n\in win_a\[\psi\]$ by item~1 of Definition~\ref{df:winning set}.

If node $n$ is a non-leaf node, then $n$ is labelled by either agent $a$ or some other agent. In the first case, by item~2 of Definition~\ref{df:winning set}, the assumption $n\in win_a(\[\phi\])$ implies the existence of $n$'s child node $m$ such that $m\in win_a(\[\phi\])$. Then, $m\in win_a(\[\psi\])$ by the induction hypothesis. Therefore, $n\in win_a(\[\psi\])$ by item~2 of Definition~\ref{df:winning set}.

In the second case, by item~3 of Definition~\ref{df:winning set}, the assumption $n\in win_a(\[\phi\])$ implies that each child node $m$ of node $n$ satisfies that $m\in win_a(\[\phi\])$. Then, by the induction hypothesis, $m\in win_a(\[\psi\])$. Therefore, $n\in win_a(\[\psi\])$ by item~3 of Definition~\ref{df:winning set}.
\end{proof}

\begin{property}\label{pp:CC=C}
For any formula $\phi\in\Phi$ and any agent $a\in\mathcal{A}$, formulae $\C_a\C_a\phi$ and $\C_a\phi$ are semantically equivalent.
\end{property}

\begin{proof}
By Definition~\ref{semantically equivalent}, it suffices to show that for any outcome $w$ in any extensive form game, $w\in\[\C_a\C_a\phi\]$ if and only if $w\in\[\C_a\phi\]$.

For the ``\textit{if}'' part, consider any outcome $w\in\[\C_a\phi\]$. By item~4 of Definition~\ref{sat}, there is a node $n$ such that $w\preceq n$ and
\begin{equation}\label{eq:26-apr-a}
    n\in win_a(\[\neg\phi\]).
\end{equation}
At the same time, $\[\phi\]\supseteq\[\C_a\phi\]$ by item~4 of Definition~\ref{sat}. Then, $\[\neg\phi\]\subseteq\[\neg\C_a\phi\]$ by item~2 of Definition~\ref{sat}. It further implies that $win_a(\[\neg\phi\])\subseteq win_a(\[\neg\C_a\phi\])$ by Lemma~\ref{lm:26-apr-a}.
Hence, $n\in win_a(\[\neg\C_a\phi\])$ by statement~\eqref{eq:26-apr-a}. Together with the statement that $w\preceq n$ and the assumption $w\in\[\C_a\phi\]$, it can be concluded that $w\in\[\C_a\C_a\phi\]$ by item~4 of Definition~\ref{sat}.

The ``\textit{only if}'' part follows directly from item~4 of Definition~\ref{sat}.
\end{proof}

Note that, {\em formulae $\S_a\S_a\phi$ and $\S_a\phi$ are not semantically equivalent}. This can be observed in outcome $w_2$ of the game depicted in Figure~\ref{fig:board and governor A}. Let formula $\phi$ represent ``Beach is left in prison''. Then, $w_2\in\[\S_g\phi\]$. However, since $w_1\notin\[\S_g\phi\]$, by item~3 of Definition~\ref{df:winning set}, the root node does not belong to the set $win_g(\[\S_g\phi\])$. Then, $w_2\notin\[\S_g\S_g\phi\]$ by item~5(a) of Definition~\ref{sat}. Hence, $\[\S_g\S_g\phi\]\neq\[\S_g\phi\]$ in the game depicted in Figure~\ref{fig:board and governor A}. Therefore, formulae $\S_a\S_a\phi$ and $\S_a\phi$ are not semantically equivalent by Definition~\ref{semantically equivalent}.

\subsection{Semantic equivalence of $\S_b\S_a\phi$ and $\bot$}\label{app_sec:SS=false}

I state the main result of this subsection as Property~\ref{pp:SS empty}. To increase the readability of its proof, I first show the next two lemmas.

\begin{lemma}\label{lm:Saphi To phi}
$\[\S_a\phi\]\subseteq\[\phi\]$ for any formula $\phi\in\Phi$, any agent $a\in\mathcal{A}$, and any extensive form game $G$.
\end{lemma}
\begin{proof}
By item~5 of Definition~\ref{sat},
$$
\[\S_a\phi\]\subseteq\Omega(G)
$$
and
$$
\[\S_a\phi\]\subseteq win_a(\[\phi\]).
$$
Hence, 
\begin{equation}\label{eq:2-Aug-1}
\[\S_a\phi\]\subseteq\Omega(G)\cap win_a(\[\phi\]).
\end{equation}
By the minimality condition in Definition~\ref{df:winning set},
\begin{equation*}\label{eq:2-Aug-2}
\Omega(G)\cap win_a(\[\phi\])= \[\phi\].
\end{equation*}
Therefore, 
$
\[\S_a\phi\]\subseteq\[\phi\]
$
by statements~\eqref{eq:2-Aug-1}.
\end{proof}


\begin{lemma}\label{lm:26-apr-b}
For any outcomes $w$ and $w'$, any formula $\phi$, and any $\[\phi\]$-achievement point $n$ by some agent, if $w\in\[\S_a\phi\]$,  $w\preceq n$, and  $w'\preceq n$, then $w'\in\[\S_a\phi\]$.
\end{lemma}

\begin{proof}
By Lemma~\ref{lm:achievement point existence} and the assumption of the lemma, node $n$ is the unique $\[\phi\]$-achievement point such that $w\preceq n$.
Then, node $n$ is the $\[\phi\]$-achievement point by agent~$a$ by item~5(b) of Definition~\ref{sat} and the assumption $w\in\[\S_a\phi\]$.
By item~3 of Definition~\ref{df:achievement point}, the statement that node $n$ is the $\[\phi\]$-achievement point implies that $w''\in \[\phi\]$ for each outcome $w''$ such that $w''\preceq n$.
Thus, by Definition~\ref{df:winning set},
\begin{equation}\label{eq:26-apr-b}
\{m\;|\; m\preceq n\}\subseteq win_a(\[\phi\]).
\end{equation}
Since $w\preceq n$ by the assumption of the lemma, the set of nodes on the path from the root to node $n$ is a subset of the set of nodes on the path from the root to outcome $w$. That is,
\begin{equation}\label{eq:27-July-1}
\{m'\;|\; n\preceq m'\}\subseteq\{m'\;|\; w\preceq m'\}.
\end{equation}
By item~5(a) of Definition~\ref{sat}, the assumption $w\in\[\S_a\phi\]$ of the lemma implies that
\begin{equation*}\label{eq:27-July-2}
\{m'\;|\; w\preceq m'\}\subseteq win_a(\[\phi\]).
\end{equation*}
Hence, by statements~\eqref{eq:27-July-1},
\begin{equation}\label{eq:26-apr-c}
\{m'\;|\; n\preceq m'\}\subseteq win_a(\[\phi\]).
\end{equation}
Note that, by the assumption $w'\preceq n$ of the lemma, every node on the path from the root to outcome $w'$ is in the set $\{m\;|\; m\preceq n\}$ or in the set $\{m'\;|\; n\preceq m'\}$. Thus, by statements~\eqref{eq:26-apr-b} and \eqref{eq:26-apr-c},
\begin{equation}\label{eq:26-apr-d}
\{m''\;|\; w'\preceq m''\} \subseteq win_a(\[\phi\])
\end{equation}
Observe that, node $n$ is also the $\[\phi\]$-achievement point by agent~$a$ on the path from the root to outcome $w'$. Therefore, $w'\in\[\S_a\phi\]$ by statement~\eqref{eq:26-apr-d} and item~5 of Definition~\ref{sat}.
\end{proof}

\begin{property}\label{pp:SS empty}
For any formula $\phi\in\Phi$ and any distinct agents $a,b\in\mathcal{A}$, formula $\S_b\S_a\phi$ is semantically equivalent to $\bot$.
\end{property}

\begin{proof}
By Definition~\ref{semantically equivalent}, it suffices to prove $\[\S_b\S_a\phi\]=\varnothing$ for each extensive form game. Moreover, $\[\S_b\S_a\phi\]\subseteq\[\S_a\phi\]$ by Lemma~\ref{lm:Saphi To phi}. Thus, it suffices to show that $w\notin\[\S_b\S_a\phi\]$ for each outcome $w\in\[\S_a\phi\]$.

Consider an outcome $w\in\[\S_a\phi\]$. Then, by item~5(b) of Definition~\ref{sat} and Definition~\ref{df:achievement point}, there is a $\[\phi\]$-achievement point $n$ by agent $a$ such that
\begin{enumerate}
    \item $w\preceq n$;
    \item $parent(n)$ is labelled with agent $a$;
    \item there exists an outcome $w'$ such that $w'\preceq parent(n)$ and $w'\notin \[\phi\]$.
\end{enumerate}
By Lemma~\ref{lm:Saphi To phi}, the statement $w'\notin \[\phi\]$ implies that
\begin{equation}\label{eq:27-July-3}
w'\notin \[\S_a\phi\].
\end{equation}
On the other hand, because $n$ is the $\[\phi\]$-achievement point such that $w\preceq n$, by Lemma~\ref{lm:26-apr-b}, the assumption $w\in\[\S_a\phi\]$ implies that
\begin{equation}\label{eq:27-July-4}
w''\in\[\S_a\phi\]
\end{equation}
for each outcome $w''$ such that $w''\preceq n$. Hence, by Definition~\ref{df:achievement point}, statements~\eqref{eq:27-July-3}, \eqref{eq:27-July-4}, and items~2 and 3 above imply that node $n$ is a $\[\S_a\phi\]$-achievement point by agent $a$. Thus, by Lemma~\ref{lm:achievement point existence} and item~1 above, there is no $\[\S_a\phi\]$-achievement point $m$ by agent $b$ such that $w\preceq m$. Therefore, $w\notin\[\S_b\S_a\phi\]$ by item~5(b) of Definition~\ref{sat}.
\end{proof}

\subsection{Semantic equivalence of $\S_b\C_a\phi$ and $\bot$}\label{app_sec:SC=false}

I state the main result of this subsection in Property~\ref{pp:SC empty}. To increase the readability of its proof, I first show the next lemma. It shows, if one agent has a strategy to achieve a statement at some node, then at the same node, another agent cannot achieve the negation of the statement.

\begin{lemma}\label{lm:win in any agents}
For any formula $\phi\in\Phi$, any node $n$ in an extensive form game, and any distinct agents $a,b\in\mathcal{A}$, if $n\in win_a(\[\phi\])$, then $n\notin win_b(\[\neg\phi\])$.
\end{lemma}
\begin{proof}
I prove this lemma by backward induction on node $n$. In other words, by induction on the depth of the subtree rooted at node $n$. 

If node $n$ is a leaf, then the assumption $n\in win_a(\[\phi\])$ of the lemma implies that $n\in\[\phi\]$ by item~1 and the minimality condition in Definition~\ref{df:winning set}. Then, $n\notin\[\neg\phi\]$ by item~2 of Definition~\ref{sat}. Therefore, $n\notin win_b(\[\neg\phi\])$ again by the minimality condition in Definition~\ref{df:winning set}.

If node $n$ is a non-leaf node, then $n$ is labelled by either agent $a$ or some other agent. In the first case, by item~2 of Definition~\ref{df:winning set}, the assumption $n\in win_a(\[\phi\])$ of the lemma implies that $m\in win_a(\[\phi\])$ for some child node $m$ of node $n$. Then, $m\notin win_b(\[\neg\phi\])$ by the induction hypothesis. Therefore, $n\notin win_b(\[\neg\phi\])$ by item~3 of Definition~\ref{df:winning set}.

In the second case, by item~3 of Definition~\ref{df:winning set}, the assumption $n\in win_a(\[\phi\])$ of the lemma implies that node $l\in win_a(\[\phi\])$ for each child node $l$ of node $n$. Hence, $l\notin win_b(\[\neg\phi\])$ by the induction hypothesis. Therefore, $n\notin win_b(\[\neg\phi\])$ by items~2 and 3 of Definition~\ref{df:winning set}.
\end{proof}

\begin{property}\label{pp:SC empty}
For any formula $\phi\in\Phi$ and any distinct agents $a,b\in\mathcal{A}$, formula $\S_b\C_a\phi$ is semantically equivalent to $\bot$.
\end{property}
\begin{proof}
By Lemma~\ref{lm:Saphi To phi} and Definition~\ref{semantically equivalent}, it suffices to show that, for any outcome $w$, if $w\in\[\C_a\phi\]$, then $w\notin\[\S_b\C_a\phi\]$.

Consider an outcome $w\in\[\C_a\phi\]$. By item~4 of Definition~\ref{sat}, there exists a node $n$ such that $w\preceq n$ and
\begin{equation}\label{eq:27-apr-a}
n\in win_a(\[\neg\phi\]).
\end{equation}
Meanwhile, $\[\phi\]\supseteq\[\C_a\phi\]$ by item~4 of Definition~\ref{sat}. Then, $\[\neg\phi\]\subseteq\[\neg\C_a\phi\]$ by item~2 of Definition~\ref{sat}. Hence, by Lemma~\ref{lm:26-apr-a},
\begin{equation}\label{eq:27-apr-b}
win_a(\[\neg\phi\])\subseteq win_a(\[\neg\C_a\phi\]).
\end{equation}
Then, $n\in win_a(\[\neg\C_a\phi\])$ by statements~\eqref{eq:27-apr-a} and \eqref{eq:27-apr-b}. Thus, $n\notin win_b(\[\C_a\phi\])$ by Lemma~\ref{lm:win in any agents}. Together with the statement $w\preceq n$ and item~5(a) of Definition~\ref{sat}, it can be concluded that $w\notin\[\S_b\C_a\phi\]$.
\end{proof}

\section{Complexity of model checking}\label{app_sec:complexity}

Denote by $|G|$ the number of nodes in game $G$ and by $|\phi|$ the size of formula $\phi$.
For a given tree (game) $G$, a subset of the nodes can be represented with a Boolean array of the size $O(|G|)$. Then, the set operations that determine whether a set contains an element, add an element to a set, or remove an element from a set take $O(1)$. The other set operations (union, intersection, difference and complement) take $O(|G|)$.

{\bf Algorithm~\ref{alg:decide achievement point}} shows the pseudocode for computing the set of all $X$-achievement points (by any agent) in game $G$. By Definition~\ref{df:achievement point}, a node is an $X$-achievement point if items~2 and 3 of Definition~\ref{df:achievement point} are satisfied. 
Algorithm~\ref{alg:decide achievement point} checks these two conditions for each node by backward induction in the tree. 
Note that, for a leaf node, item~3 of Definition~\ref{df:achievement point} is satisfied if the node itself is in set $X$; for a non-leaf node, this item is satisfied if it holds in all children of this node. For any node, item~2 of Definition~\ref{df:achievement point} is satisfied if item~3 does {\em not} hold in at least one of its siblings (including itself). Hence, backward induction can be processed in the {\em reversed breadth-first-search order} of nodes.

In Algorithm~\ref{alg:decide achievement point}, lines~\ref{line:alg 1.1} to \ref{line:alg 1.9} use a first-in-first-out queue $Q$ to push all the non-leaf nodes into a first-in-last-out stack $S$ in the {\em breadth-first-search order}.
Line~\ref{line:alg 1.10} defines a set $A$ to store all the nodes satisfying item~3 of Definition~\ref{df:achievement point}.
Line~\ref{line:alg 1.11} defines a set $W$ to store all the $X$-achievement points.
The while-loop starting at line~\ref{line:alg 1.12} is the core of this algorithm. 
Line~\ref{line:alg 1.13} pops the non-leaf nodes out of stack $S$ in the reversed breadth-first-search order. For each non-leaf node $n$, lines~\ref{line:alg 1.14} to \ref{line:alg 1.20} check if it satisfies item~3 of Definition~\ref{df:achievement point}. 
That is to check if every child of node $n$ satisfies item~3 of Definition~\ref{df:achievement point}.
At the same time, the children of $n$ that satisfy this item are included in a temporal set $T$ by line~\ref{line:alg 1.18}. 
Then, in lines~\ref{line:alg 1.21} to \ref{line:alg 1.24}, node $n$ is added to set $A$ if it satisfies item~3 of Definition~\ref{df:achievement point}. Otherwise, each node in set $T$ satisfies item~2 of Definition~\ref{df:achievement point} because one of their siblings does not satisfy item~3 of Definition~\ref{df:achievement point}. Since the nodes in set $T$ also satisfy item~3 of the same definition, they are $X$-achievement points and thus added to set $W$ in line~\ref{line:alg 1.24}.

Note that, each node has at most one parent node. Hence, the for-loops starting at lines~\ref{line:alg 1.7} and \ref{line:alg 1.16} are executed at most $O(|G|)$ times in total. Also, line~\ref{line:alg 1.24} takes $O(|G|)$ in total because the elements of set $T$ are distinct nodes in each loop.
Therefore, the whole process in Algorithm~\ref{alg:decide achievement point} takes $O(|G|)$.

\begin{algorithm}[hbt] 
\caption{Computing the $X$-achievement points}
\label{alg:decide achievement point}
\DontPrintSemicolon
\KwIn{game $G$, set $X$;}
\KwOut{the set of all $X$-achievement points.}
$Q\leftarrow$ an empty (FIFO) queue\label{line:alg 1.1}\;
$S\leftarrow$ an empty (FILO) stack\label{line:alg 1.2}\;
$Q.enqueue(\text{the root node})$\;
\While(){$Q$ is not empty}{
    $n\leftarrow Q.dequeue$\;
    \If{$n$ is a non-leaf node}{
      \For({\scriptsize\tcp*[f]{$O(|G|)$ total}\normalsize}){each $m\in Child(n)$\label{line:alg 1.7}}{
            $Q.enqueue(m)$\; 
        }
      $S.push(n)$\label{line:alg 1.9}\;
    }
}
$A\leftarrow X$\label{line:alg 1.10}\;
$W\leftarrow\varnothing$\label{line:alg 1.11}\;
\While(){$S$ is not empty\label{line:alg 1.12}}{
    $n\leftarrow S.pop\label{line:alg 1.13}$\;
    $nInA \leftarrow true$\label{line:alg 1.14}\;
    $T\leftarrow\varnothing$\label{line:alg 1.15}\;
    \For({\scriptsize\tcp*[f]{$O(|G|)$ total}\normalsize}){each $m\in Child(n)$\label{line:alg 1.16}}{
        \eIf{$m\in A$}{
            $T.add(m)$\label{line:alg 1.18}\;
        }{
            $nInA\leftarrow false$\label{line:alg 1.20}\;
        }
    }
    \eIf(){$nInA$\label{line:alg 1.21}}{
        $A.add(n)$\label{line:alg 1.22}\;
    }{
        $W\leftarrow W\cup T$\label{line:alg 1.24}{\scriptsize\tcp*{$O(|G|)$ total}\normalsize}
    }
}
\Return{$W$}\;
\end{algorithm}

{\bf Algorithm~\ref{alg:win set}} shows the pseudocode for computing the set $win_a(X)$ in game $G$, given an agent $a$ and a set $X$ of outcomes.
By Definition~\ref{df:winning set}, the computation of the set $win_a(X)$ uses backward induction in the tree. A leaf node is in the set $win_a(X)$ if it is in set $X$; a non-leaf node is in the set $win_a(X)$ if it satisfies either item~2 or item~3 of Definition~\ref{df:winning set}.
Algorithm~\ref{alg:win set} checks if the above statements are satisfied for each node in the reversed breadth-first-search order.

In Algorithm~\ref{alg:win set}, lines~\ref{line:alg 2.1} to \ref{line:alg 2.9} are the same as that in Algorithm~\ref{alg:decide achievement point}. 
Line~\ref{line:alg 2.10} defines a set $W$ to store all the nodes in the set $win_a(X)$ and initialize it with the set $X$ according to item~1 of Definition~\ref{df:winning set}.
The core of this algorithm is the while-loop starting at line~\ref{line:alg 2.11}. For each non-leaf node, lines~\ref{line:alg 2.14} to \ref{line:alg 2.17} check if item~2 of Definition~\ref{df:winning set} is satisfied, while lines~\ref{line:alg 2.19} to \ref{line:alg 2.23} check if item~3 of the same definition is satisfied.

Since each node has at most one parent node, the for-loops starting at lines~\ref{line:alg 2.14} and \ref{line:alg 2.20} shall be executed at most $O(|G|)$ times in all the while-loops starting at line~\ref{line:alg 2.11}. This same is true for the for-loop starting at line~\ref{line:alg 2.7}. Therefore, the whole process in Algorithm~\ref{alg:win set} takes $O(|G|)$.

\begin{algorithm}[hbt]
\caption{Computing $win_a(X)$}
\label{alg:win set}
\DontPrintSemicolon
\KwIn{game $G$, agent $a$, set $X$;}
\KwOut{$win_a(X)$.}
$Q\leftarrow$ an empty (FIFO) queue\label{line:alg 2.1}\;
$S\leftarrow$ an empty (FILO) stack\label{line:alg 2.2}\;
$Q.enqueue(\text{the root node})$\label{line:alg 2.3}\;
\While(){$Q$ is not empty\label{line:alg 2.4}}{
    $n\leftarrow Q.dequeue$\;
    \If{$n$ is a non-leaf node}{
      \For({\scriptsize\tcp*[f]{$O(|G|)$ total}\normalsize}){each $m\in Child(n)$\label{line:alg 2.7}}{
            $Q.enqueue(m)$\; 
        }
      $S.push(n)$\label{line:alg 2.9}\;
    }
}
$W\leftarrow X$\label{line:alg 2.10}\;
\While(){$S$ is not empty\label{line:alg 2.11}}{
    $n\leftarrow S.pop$\;
    \eIf{$n$ is labelled with $a$\label{line:alg 2.13}}{
        \For({\scriptsize\tcp*[f]{$O(|G|)$ total}\normalsize}){each $m\in Child(n)$\label{line:alg 2.14}}{
            \If(){$m\in W$}{
                $W.add(n)$\;
                break\label{line:alg 2.17}\;
            }
        }
    }{
        $W.add(n)$\label{line:alg 2.19}\;
        \For({\scriptsize\tcp*[f]{$O(|G|)$ total}\normalsize}){each $m\in Child(n)$\label{line:alg 2.20}}{
            \If(){$m\notin W$}{
                 $W.remove(n)$\;
                 break\label{line:alg 2.23}\;
            }
        }
    }
}
\Return{$W$}
\end{algorithm}

For any formula $\phi\in\Phi$, {\bf the calculation of the truth set $\[\phi\]$ uses a recursive process on the structural complexity of formula $\phi$}, which is the standard way of model checking.
Specifically, if formula $\phi$ is a propositional variable $p$, then the computation of the set $\[\phi\]$ takes $O(|G|)$ by item~1 of Definition~\ref{sat}\footnote{The statement is true under the assumption that checking if an outcome is labelled with a propositional variable takes constant time. I state this assumption in Subsection~\ref{sec:complexity} of the main text.}.
If formula $\phi$ has the form $\neg\psi$ or $\psi\wedge\psi'$, then, given the sets $\[\psi\]$ and $\[\psi'\]$, the computation of the set $\[\phi\]$ takes $O(|G|)$ by items~2 and 3 of Definition~\ref{sat}.

\begin{algorithm}[ht]
\caption{Computing $\[\C_a\psi\]$}
\label{alg:Ca}
\DontPrintSemicolon
\KwIn{game $G$, agent $a$, set $\[\psi\]$;}
\KwOut{$\[\C_a\psi\]$.}
$\[\neg\psi\]\leftarrow \Omega(G)\setminus \[\psi\]$\label{line:alg 3.1}\;
$win_a\[\neg\psi\]\leftarrow$ Algorithm~\ref{alg:win set}\label{line:alg 3.2}{\scriptsize\tcp*{$O(|G|)$}\normalsize}
$W\leftarrow\varnothing$\label{line:alg 3.3}\;
$Q\leftarrow$ an empty (FIFO) queue\label{line:alg 3.4}\;
$Q.enqueue(\text{the root node})$\label{line:alg 3.5}\;
\While(){$Q$ is not empty\label{line:alg 3.6}}{
    $n\leftarrow Q.dequeue$\label{line:alg 3.7}\;
    \eIf(){$n\in win_a\[\neg\psi\]$\label{line:alg 3.8}}{
        $T\leftarrow\{w\in\Omega(G)\,|\,w\preceq n\}$\label{line:alg 3.9}{\scriptsize\tcp*{$O(|G|)$ total}\normalsize}
        $W\leftarrow W\cup T$\label{line:alg 3.10} {\scriptsize\tcp*{$O(|G|)$ total}\normalsize}
    }{
        \For({\scriptsize\tcp*[f]{$O(|G|)$ total}\normalsize}){each $m\in Child(n)$\label{line:alg 3.12}}{
            $Q.enqueue(m)$\;\label{line:alg 3.13}
        }
    }
}
\Return{$W\cap\[\psi\]$}\;
\end{algorithm}

If formula $\phi$ has the form $\C_a\psi$, then, by item~4 of Definition~\ref{sat}, the set $\[\C_a\psi\]$ can be computed with {\bf Algorithm~\ref{alg:Ca}}. The key of this algorithm is to compute the set $W$ of all {\em outcomes} to which the path of play contains a node in the set $win_a\[\neg\psi\]$ (\textit{i.e.} agent $a$ has a strategy to prevent $\psi$).
Algorithm~\ref{alg:Ca} checks for each node if it is in the set $win_a\[\neg\psi\]$ in the {\em breadth-first-search} order. Once a node $n\in win_a\[\neg\psi\]$ is found, all the descendant outcomes of $n$ are put into set $W$.
Note that, the nodes in the subtree rooted at such node $n$ need {\em not} to be checked any more. This is the pruning strategy to increase the efficiency of the algorithm. 

In Algorithm~\ref{alg:Ca}, lines~\ref{line:alg 3.1} and \ref{line:alg 3.2} compute the set $win_a(\[\neg\psi\])$ using Algorithm~\ref{alg:win set}. Line~\ref{line:alg 3.3} defines the set $W$ described in the above paragraph. The first-in-first-out queue $Q$ defined in line~\ref{line:alg 3.4} is used for getting the breadth-first-search order of the nodes. 
The while-loop starting at line~\ref{line:alg 3.6} is the core of this algorithm. The if-condition in line~\ref{line:alg 3.8} checks for each node $n$ if it belongs to the set $win_a(\[\neg\psi\])$. If so, then, on the path of play to each descendant outcome of $n$, node $n$ is the place where agent $a$ has a strategy to prevent $\psi$. Hence, lines~\ref{line:alg 3.9} and \ref{line:alg 3.10} put these descendant outcomes into set $W$. Note that, in this case, the subtree rooted at node $n$ is pruned. 
If the if-condition in line~\ref{line:alg 3.8} is not satisfied, then, due to the pruning strategy, no node on the path of play from the root to node $n$ belongs to the set $win_a(\[\neg\psi\])$. 
Thus, in order to check if the descendant outcomes of node $n$ belong to set $W$, all children of node $n$ are added to the queue by lines~\ref{line:alg 3.12} and \ref{line:alg 3.13}. 
Finally, after computing set $W$ with the while-loop, the algorithm returns the intersection of sets $W$ and $\[\psi\]$ following item~4 of Definition~\ref{sat}.

Note that, each execution of lines~\ref{line:alg 3.9} and \ref{line:alg 3.10} takes linear time in terms of the size of the subtree rooted at node $n$. Thus, the execution of these two lines takes $O(|G|)$ in total. Also, because each node has at most one parent, the for-loop starting at line~\ref{line:alg 3.12} is executed at most $O(|G|)$ times in total. Therefore, the whole process in Algorithm~\ref{alg:Ca} takes $O(|G|)$.

\begin{algorithm}[ht]
\caption{Computing $\[\S_a\psi\]$}
\label{alg:Sa}
\DontPrintSemicolon
\KwIn{game $G$, agent $a$, set $\[\psi\]$;}
\KwOut{$\[\S_a\psi\]$.}
\If({\scriptsize\tcp*[f]{$O(|G|)$}\normalsize}){$\[\psi\]==\Omega(G)$\label{line:alg 4.1}}{
    \Return{$\varnothing$}\label{line:alg 4.2}\;
}
$achieve(\[\psi\])\leftarrow$ Algorithm~\ref{alg:decide achievement point}\label{line:alg 4.3}{\scriptsize\tcp*{$O(|G|)$}\normalsize}
$win_a(\[\psi\])\leftarrow$ Algorithm~\ref{alg:win set}\label{line:alg 4.4}{\scriptsize\tcp*{$O(|G|)$}\normalsize}
$W\leftarrow\varnothing$\label{line:alg 4.5}\;
$Q\leftarrow$ an empty (FIFO) queue\label{line:alg 4.6}\;
$Q.enqueue(\text{the root node})$\label{line:alg 4.7}\;
\While(){$Q$ is not empty\label{line:alg 4.8}}{
    $n\leftarrow Q.dequeue$\label{line:alg 4.9}\;
    \If{$n\notin win_a(\[\psi\])$\label{line:alg 4.10}}{
        continue\;
    }
    \eIf(){$n\in achieve(\[\psi\])$\label{line:alg 4.12}}{
        \If(){$parent(n)$ is labelled by $a$\label{line:alg 4.13}}{
            $T\leftarrow\{w\in\Omega(G)\,|\,w\preceq n\}$\label{line:alg 4.14}\!\!\!\! {\scriptsize\tcp*{$O(|G|)$ total}\normalsize}
            $W\leftarrow W\cup T$\label{line:alg 4.15} {\scriptsize\tcp*{$O(|G|)$ total}\normalsize}
        }
    }{
        \For({\scriptsize\tcp*[f]{$O(|G|)$ total}\normalsize}){each $m\in Child(n)$\label{line:alg 4.17}}{
        $Q.enqueue(m)$\label{line:alg 4.18}\;
        }
    }
}
\Return{$W$}\;
\end{algorithm}

If formula $\phi$ has the form $\S_a\psi$, then, by item~5 of Definition~\ref{sat}, the set $\[\S_a\psi\]$ can be computed with {\bf Algorithm~\ref{alg:Sa}}. 
Consider a $\[\psi\]$-achievement point $n$. Then, by Definition~\ref{df:winning set} and item~3 of Definition~\ref{df:achievement point}, every node in the subtree rooted at node $n$ belongs to the set $win_a\[\psi\]$. Thus, by item~5 of Definition~\ref{sat}, outcome $w\in\[\S_a\psi\]$ if and only if
\begin{enumerate}[label=C\arabic*., ref=C\arabic*,left=0pt]
\item there is a $\[\psi\]$-achievement point $n$ by agent $a$ such that $w\preceq n$;\label{item:9-Aug-1}
\item $m\in win_a\[\psi\]$ for each node $m$ such that $n\preceq m$.\label{item:9-Aug-2}
\end{enumerate}
Hence, if $n$ is a node satisfying the above two conditions in terms of any outcome, then all the descendant outcomes of node $n$ belong to the set $\[\S_a\psi\]$. 
Algorithm~\ref{alg:Sa} searches for all such node $n$ in the tree in the {\em breadth-first-search} order. In this process, two pruning strategies are used:
\begin{enumerate}[label=S\arabic*., ref=S\arabic*,left=0pt]
\item if node $n\notin win_a\[\psi\]$, then the subtree rooted at node $n$ is pruned because no node in this subtree satisfies condition~\ref{item:9-Aug-2} above;\label{item:9-Aug-3}
\item if node $n$ is a $\[\psi\]$-achievement by some agent other than $a$, then the subtree rooted at node $n$ is pruned because no node in this subtree is a $\[\psi\]$-achievement by agent $a$ by Lemma~\ref{lm:achievement point existence} and thus condition~\ref{item:9-Aug-1} above is not satisfied for any node in this subtree.\label{item:9-Aug-4}
\end{enumerate}
In order for simplicity, Algorithm~\ref{alg:Sa} deals with the trivial condition separately in lines~\ref{line:alg 4.1} and \ref{line:alg 4.2}.


For the non-trivial conditions, in Algorithm~\ref{alg:Sa}, lines~\ref{line:alg 4.3} and \ref{line:alg 4.4} compute the sets of $\[\psi\]$-achievement points and $win_a(\[\psi\])$ using Algorithm~\ref{alg:decide achievement point} and Algorithm~\ref{alg:win set}. Line~\ref{line:alg 4.5} defines the set $W$ to collect all the nodes in the set $\[\S_a\psi\]$ (\textit{i.e.} the descendant outcomes of nodes satisfying the above conditions~\ref{item:9-Aug-1} and \ref{item:9-Aug-2} in terms of any outcome). The first-in-first-out queue $Q$ defined in line~\ref{line:alg 4.6} is used for getting the breadth-first-search order of the nodes.
The core of this algorithm is the while-loop starting at line~\ref{line:alg 4.8}. Line~\ref{line:alg 4.10} checkes for each node $n$ if it does {\em not} belong to the set $win_a\[\psi\]$. 
If so, then the subtree rooted at $n$ is pruned by strategy~\ref{item:9-Aug-3}. Otherwise, due to strategy~\ref{item:9-Aug-3}, node $n$ satisfies condition~\ref{item:9-Aug-2} above.
Then, line~\ref{line:alg 4.12} checks if node $n$ is a $\[\psi\]$-achievement point. In particular, if node $n$ is a $\[\psi\]$-achievement point by agent $a$ as checked by line~\ref{line:alg 4.13}, then $n$ satisfies both conditions~\ref{item:9-Aug-1} and \ref{item:9-Aug-2} in terms of any of its descendant outcome. Thus, all the descendant outcomes of node $n$ are collected by set $W$. Hence, the subtree rooted at node $n$ is pruned. 
If node $n$ is a $\[\psi\]$-achievement by some agent other than $a$, then the subtree rooted at $n$ is pruned by strategy~\ref{item:9-Aug-4}.
However, if node $n$ is not a $\[\psi\]$-achievement point, then node $n$ does not satisfy condition~\ref{item:9-Aug-1} above. Since condition~\ref{item:9-Aug-2} is still satisfied for node $n$, some node in the subtree rooted at $n$ might satisfy both conditions~\ref{item:9-Aug-1} and \ref{item:9-Aug-2}. Thus, the children of node $n$ are added to the queue for further search.

In Algorithm~\ref{alg:Sa}, each execution of lines~\ref{line:alg 4.14} and \ref{line:alg 4.15} takes linear time in terms of the size of the subtree rooted at node $n$. Then, the execution of these two lines takes $O(|G|)$ in total. Also, because each node has at most one parent, the for-loop starting at line~\ref{line:alg 4.17} is executed at most $O(|G|)$ times in total. Therefore, the whole process in Algorithm~\ref{alg:Sa} takes $O(|G|)$.

In conclusion, for any formula $\phi\in\Phi$, the calculation of the set $\[\phi\]$ takes $O(|\phi|\cdot |G|)$.

\section{Proof toward responsibility gap}

\subsection{Proof of Lemma~\ref{lm:win in two agent}}\label{app_sec:proof lm win in two agent}

\textbf{Lemma \ref{lm:win in two agent}. }\textit{For any formula $\phi\in\Phi$ and any node $n$ in a two-agent extensive form game between agents $a$ and $b$, if $n\notin win_a(\[\phi\])$, then $n\in win_b(\[\neg\phi\])$.}

\vspace{1mm}
\begin{proof}
I prove this lemma by backward induction on node $n$. In other words, by induction on the depth of the subtree rooted at node $n$. 

If node $n$ is a leaf, then the assumption $n\notin win_a(\[\phi\])$ of the lemma implies that $n\notin\[\phi\]$ by item~1 of Definition~\ref{df:winning set}. This further implies $n\in\[\neg\phi\]$ by item~2 of Definition~\ref{sat}. Therefore, $n\in win_b(\[\neg\phi\])$ by item~1 of Definition~\ref{df:winning set}.

If $n$ is a non-leaf node, then $n$ is labelled by either agent $a$ or agent $b$. In the first case, by item~2 of Definition~\ref{df:winning set}, the assumption $n\notin win_a(\[\phi\])$ of the lemma implies that $m\notin win_a(\[\phi\])$ for each child node $m$ of node $n$. Then, $m\in win_b(\[\neg\phi\])$ by the induction hypothesis. Therefore, $n\in win_b(\[\neg\phi\])$ by item~3 of Definition~\ref{df:winning set}.

In the second case, by item~3 of Definition~\ref{df:winning set}, the assumption $n\notin win_a(\[\phi\])$ of the lemma implies that there is a child node $m$ of node $n$ such that $m\notin win_a(\[\phi\])$. Hence, $m\in win_b(\[\neg\phi\])$ by the induction hypothesis. Therefore, $n\in win_b(\[\neg\phi\])$ by item~2 of Definition~\ref{df:winning set}.
\end{proof}

\subsection{Instance of higher-order responsibility gap}\label{app_sec:gap instance}

\begin{figure}
\centering
\scalebox{0.5}{
\includegraphics{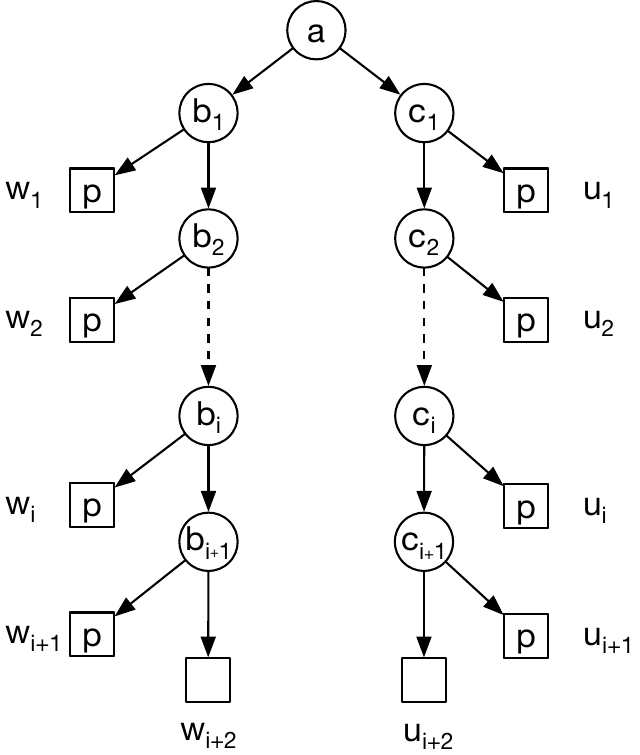}}
\caption{A game in which the set $\[\Gcs_i(p)\]$ is not empty}
\label{fig:higher order gap}
\end{figure}

For any integer $i\ge 2$, I show that in the extensive form game depicted in Figure~\ref{fig:higher order gap}, the set $\[\Gcs_i(p)\]$ is not empty. In this game, there are
\begin{itemize}
\item $2i+3$ distinct agents: $a$, $b_1$, \dots, $b_{i+1}$, $c_1$, \dots, $c_{i+1}$;
\item $2i+3$ non-leaf nodes;
\item $2i+4$ outcomes: $w_1$, \dots, $w_{i+2}$, $u_1$, \dots, $u_{i+2}$.
\end{itemize}
All the outcomes except for $w_{i+2}$ and $u_{i+2}$ are labelled with propositional variable $p$. Then, by statement~\eqref{Gcs definition},
$$
\[\Gcs_0(p)\]=\[p\]=\{w_1,\dots,w_{i+1},u_1,\dots,u_{i+1}\}.
$$
Note that, no agent in this game has an upfront strategy to guarantee that the game ends with an outcome where $p$ is true. That is, the root node does not belong to the set $win_g(\[p\])$ for each agent $g$ in this game. Thus, by item~5(a) of Definition~\ref{sat}, no agent is responsible for seeing to $p$ in any of the outcomes.
At the same time, in all the outcomes where $p$ is true, only $w_{i+1}$ and $u_{i+1}$ can be prevented (by agents $b_{i+1}$ and $c_{i+1}$, respectively). Hence, by item~4 of Definition~\ref{sat}, no agent is counterfactually responsible for $p$ in outcomes $w_1$, \dots, $w_{i}$, $u_1$, \dots, $u_{i}$. Therefore, by statement~\eqref{Gcs definition},
$$
\[\Gcs_1(p)\]=\{w_1,\dots,w_{i},u_1,\dots,u_{i}\}.
$$

With the same reasoning process, I can further deduce the following statements:
\begin{align*}
& \[\Gcs_2(p)\]=\{w_1,\dots,w_{i-1},u_1,\dots,u_{i-1}\};\\
& \[\Gcs_3(p)\]=\{w_1,\dots,w_{i-2},u_1,\dots,u_{i-2}\};\\
& \hspace{6mm} \vdots \\
& \[\Gcs_k(p)\]=\{w_1,\dots,w_{i+1-k},u_1,\dots,u_{i+1-k}\};\\
& \hspace{6mm} \vdots \\
& \[\Gcs_{i}(p)\]=\{w_1,u_1\}.
\end{align*}

\subsection{Proof of Corollary~\ref{cr:no higher-order Gcs gap}}\label{app_sec:corollary no Gcs gap}


\vspace{1mm}
Given Theorem~\ref{th:no higher-order Gc}, in order to prove Corollary~\ref{cr:no higher-order Gcs gap}, it suffices to show the statement in the next lemma is true.

\begin{lemma}
$\[\Gcs_i(\phi)\]\subseteq \[\Gc_i(\phi)\]$.    
\end{lemma}
\begin{proof}
I prove the lemma by induction. If $i=0$, then $\Gcs_i(\phi)=\phi=\Gc_i(\phi)$ by statements~\eqref{Gcs definition} and \eqref{Gc definition}. Therefore, $\[\Gcs_i(\phi)\]= \[\Gc_i(\phi)\]$.

\vspace{1mm}
In the cases where $i\ge 1$, by statements~\eqref{Gcs definition}, \eqref{Gc definition} and item~3 of Definition~\ref{sat}, 
\begin{equation}\label{18-apr-a}
\[\Gcs_i(\phi)\]\subseteq \[\Gcs_{i-1}(\phi)\]\cap \bigcap_{a\in\mathcal{A}}\[\neg\C_a\Gcs_{i-1}(\phi)\]
\end{equation}
and
\begin{equation}\label{18-apr-b}
\[\Gc_i(\phi)\]=\[\Gc_{i-1}(\phi)\]\cap \bigcap_{a\in\mathcal{A}}\[\neg\C_a\Gc_{i-1}(\phi)\].
\end{equation}
Consider any outcome $w\in \[\Gcs_i(\phi)\]$. It suffices to show that $w\in \[\Gc_i(\phi)\]$. Note that, by statement~(\ref{18-apr-a}) and item~3 of Definition~\ref{sat}, the assumption $w\in \[\Gcs_i(\phi)\]$ implies that 
\begin{equation}\label{18-apr-c}
    w\in \[\Gcs_{i-1}(\phi)\]
\end{equation}
and
\begin{equation}\label{18-apr-d}
    w\in \bigcap_{a\in\mathcal{A}}\[\neg\C_a\Gcs_{i-1}(\phi)\].
\end{equation}
By the induction hypothesis, statement~\eqref{18-apr-c} implies that $w\in \[\Gc_{i-1}(\phi)\]$. Thus, by statement~\eqref{18-apr-b} and item~3 of Definition~\ref{sat}, in order to show that $w\in\[\Gc_i(\phi)\]$, it suffices to prove that $w\in\[\neg\C_a\Gc_{i-1}(\phi)\]$ for each agent $a\in\mathcal{A}$.


Towards a contradiction, suppose $w\notin\[\neg\C_a\Gc_{i-1}(\phi)\]$ for some agent $a$. 
Then, $w\in\[\C_a\Gc_{i-1}(\phi)\]$ by item~2 of Definition~\ref{sat}. Hence, by item 4 of Definition \ref{sat}, there exists a node $n$ such that
\begin{equation}\label{eq:30-July-3}
w\preceq n
\end{equation}
and 
\begin{equation}\label{eq:30-July-1}
n\in win_a(\[\neg\Gc_{i-1}(\phi)\]).
\end{equation}
At the same time, 
$
\[\neg\Gc_{i-1}(\phi)\]\subseteq \[\neg\Gcs_{i-1}(\phi)\]
$
by the induction hypothesis and item~2 of Definition~\ref{sat}.
Thus, by Lemma~\ref{lm:26-apr-a} and statements~\eqref{eq:30-July-1},
$$
n\in win_a(\[\neg\Gcs_{i-1}(\phi)\]).
$$ 
Hence, $w\in\[\C_a\Gcs_{i-1}(\phi)\]$ by statements~\eqref{18-apr-c}, \eqref{eq:30-July-3}, and item~4 of Definition~\ref{sat}, which contradicts statement~\eqref{18-apr-d}.
\end{proof}

\end{document}